\documentclass[10pt,twocolumn,letterpaper]{article}

\usepackage{cvpr}              
\definecolor{cvprblue}{rgb}{0.21,0.49,0.74}
\usepackage[pagebackref,breaklinks,colorlinks,allcolors=cvprblue]{hyperref}
\usepackage{times}

\usepackage{amsmath,amsfonts,bm}

\def\eqref#1{equation~\ref{#1}}

\def\1{\bm{1}}

\DeclareMathAlphabet{\mathsfit}{\encodingdefault}{\sfdefault}{m}{sl}
\SetMathAlphabet{\mathsfit}{bold}{\encodingdefault}{\sfdefault}{bx}{n}

\usepackage{kotex}
\usepackage{url}            %
\usepackage{booktabs}       %
\usepackage{nicefrac}       %
\usepackage{microtype}      %
\usepackage[figuresright]{rotating}
\usepackage{amsfonts}
\usepackage{amssymb}
\usepackage{amsthm}
\usepackage{amsmath}
\usepackage{multirow}
\usepackage{array}
\usepackage{graphicx}
\usepackage[table]{xcolor}
\usepackage{arydshln}
\usepackage{siunitx}
\usepackage{subcaption}
\usepackage{tcolorbox}
\usepackage{bbm}
\usepackage[english]{babel}
\newtheorem{theorem}{Theorem}
\newtheorem{proposition}{Proposition}
\newtheorem{corollary}{Corollary}
\newtheorem{lemma}{Lemma}
\usepackage{algorithm}
\usepackage[noend]{algpseudocode} 

\newcommand{\result}[2]{ $\displaystyle #1$ \color{darkgray}{\scriptsize{${#2}$}}}

\algnewcommand{\Parameters}[1]{\State \textbf{Parameters:} #1}

\def\paperID{} 
\def\confName{CVPR}
\def\confYear{2026}

\title{\emph{Evidential Transformation Network}: Turning Pretrained Models into \\Evidential Models for Post-hoc Uncertainty Estimation}

\author{
Yongchan Chun$^{1}$\thanks{This work was done while the author was at Korea University.} \quad
Chanhee Park$^{1}$ \quad
Jeongho Yoon$^{1}$ \quad
Jaehyung Seo$^{2}$\thanks{Corresponding authors.} \quad
Heuiseok Lim$^{1}$\footnotemark[2]\\[0.5em]
$^{1}$Korea University \qquad
$^{2}$Konkuk University\\[0.5em]
{\tt\small \{cyc9805, pch7678, aa007878, limhseok\}@korea.ac.kr, seojae777@konkuk.ac.kr}
}
\begin{document}
\maketitle
\begin{abstract}
Pretrained models have become standard in both vision and language, yet they typically do not provide reliable measures of confidence. Existing uncertainty estimation methods—such as deep ensembles and MC dropout—are often too computationally expensive to deploy in practice. Evidential Deep Learning (EDL) offers a more efficient alternative, but it requires models to be trained to output evidential quantities from the start, which is rarely true for pretrained networks.
To enable EDL-style uncertainty estimation in pretrained models, we propose the Evidential Transformation Network (ETN), a lightweight post-hoc module that converts a pretrained predictor into an evidential model. ETN operates in logit space: it learns a sample-dependent affine transformation of the logits and interprets the transformed outputs as parameters of a Dirichlet distribution for uncertainty estimation.
We evaluate ETN on image classification and large language model question-answering benchmarks, under both in-distribution and out-of-distribution settings. ETN consistently improves uncertainty estimation over post-hoc baselines, while preserving accuracy and adding only minimal computational overhead. Our code is available at \url{https://github.com/cyc9805/Evidential-Transformation-Network}.
\end{abstract}

\section{Introduction}
As probabilistic deep learning models and datasets scale to capture increasingly complex patterns, training from scratch has become extremely expensive~\citep{kaplan2020scalinglawsneurallanguage}. Consequently, deep learning community has widely adopted a \emph{pretrain–then–finetune} strategy, adapting publicly available pretrained models to downstream tasks.

While using pretrained deep learning models is both effective and cost-efficient, a key question remains: \emph{to what extent can we trust a model’s prediction?} In other words, we seek to quantify how certain the model is about its output. Uncertainty estimation~\citep{gawlikowski2023survey} addresses this by modeling a \emph{second-order} distribution, i.e, a distribution over the predictive categorical probabilities\cite{juergens2024epistemic,bengs2023secondorderscoringrulesepistemic}.

As constructing the full second-order distribution is intractable, it is typically approximated via methods such as Deep Ensembles~\citep{NIPS2017_9ef2ed4b}, MC Dropout~\citep{pmlr-v48-gal16}, and Laplace Approximation~\citep{daxberger2021laplace}. However, these techniques demand substantial compute—training multiple models, performing many stochastic forward passes, or estimating Hessians of model parameters—making them difficult to deploy with large-scale pretrained models in practice.

\begin{figure}[t]
\centering
\includegraphics[width=\linewidth]{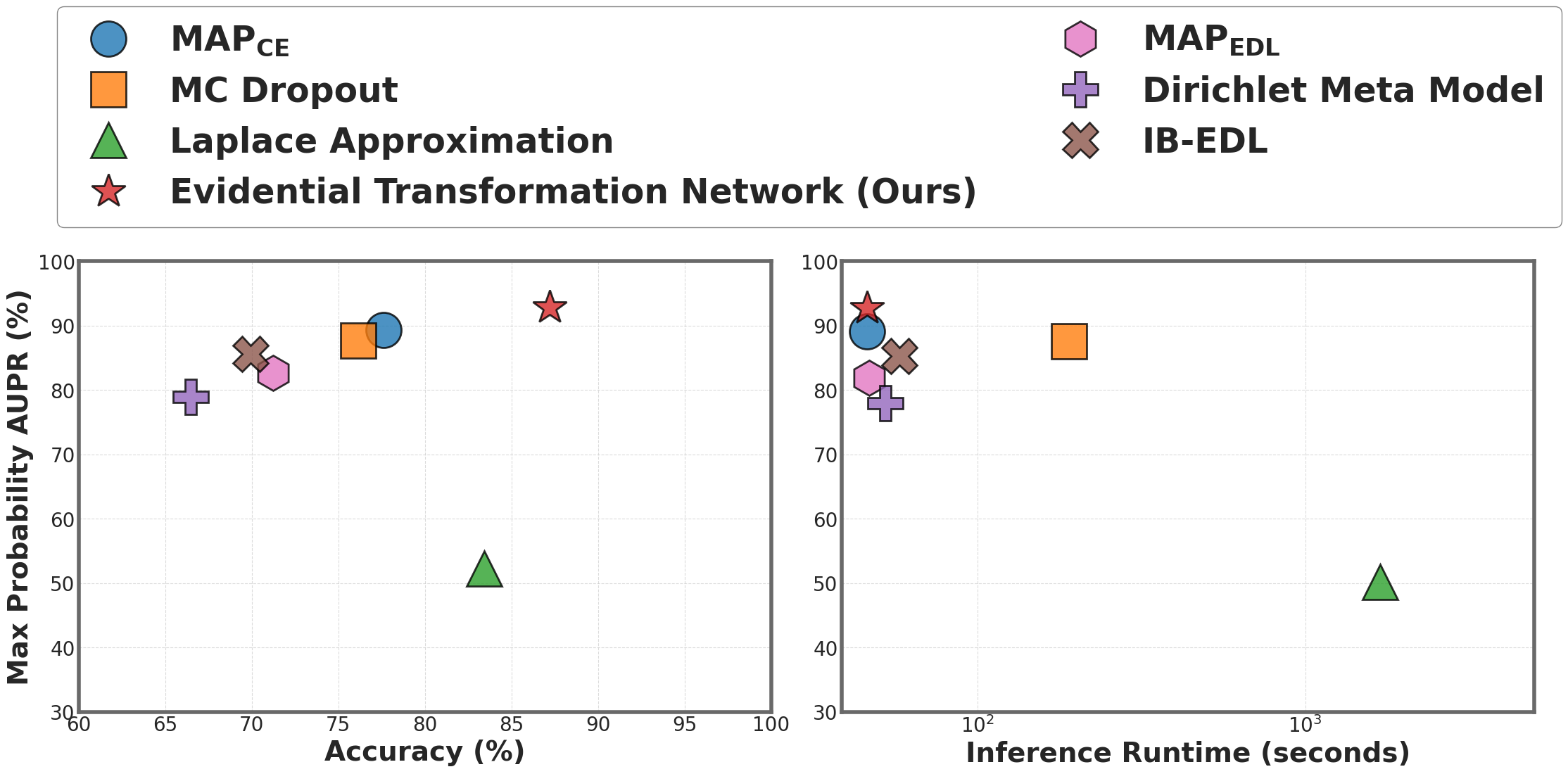}
\caption{
Comparison of average uncertainty estimation performance across in-distribution (ID) and out-of-distribution (OOD) settings. For ID, we score AUPR for correctly identifying confident predictions, while for OOD, AUPR for separating ID from OOD samples by confidence is scored.
Plots show performance against \textbf{accuracy} (left) and \textbf{inference runtime} (right).
\textbf{Evidential Transformation Network} achieves the best uncertainty performance with almost no additional inference cost and no accuracy drop.}
\label{fig:frontier}
\end{figure}


\emph{Evidential Deep Learning} (EDL)~\citep{sensoy2018evidentialdeeplearningquantify} reduces the cost of uncertainty estimation by modeling the second-order distribution as a Dirichlet distribution and training the model to predict its parameters directly. As this requires no additional uncertainty-specific components, EDL adds no inference-time overhead, enabling lightweight uncertainty estimation.

In this work, to exploit the advantages of EDL, we aim to \emph{transform} standard pretrained models into pretrained EDL models.
Since only a small dataset is typically available for such post-hoc adaptation, naïve fine-tuning risks overfitting, potentially degrading the pretrained representations.
To avoid this, we seek a transformation that minimally disturbs the model’s learned representations.
Specifically, we operate in the \emph{logit space}, applying an affine transformation to the logits using a \emph{transformation parameter}, such that the transformed logits can serve as valid Dirichlet parameters for EDL-style uncertainty estimation.

At first glance, our approach resembles Platt scaling~\citep{platt1999probabilistic,niculescu2005predicting,guo2017calibration}, which learns parameters to calibrate logits into well-calibrated probabilities.
However, unlike standard Platt scaling, which uses a single static scaling factor for all samples, our approach employs \emph{sample-dependent} parameters that adapt to each input.
We further motivate and describe this adaptive transformation in a later section.

Based on this idea, we introduce \textbf{Evidential Transformation Network (ETN)}—a lightweight module that generates sample-specific transformation parameters to convert pretrained models into evidential ones. We apply ETN to both image classification models and large language models (LLMs), successfully turning them into evidential models that produce more reliable uncertainty estimates without sacrificing accuracy and inference cost (see Figure~\ref{fig:frontier}).

In summary, our contributions are:
\begin{itemize}
\item We propose \textbf{Evidential Transformation Network (ETN)}, a simple, lightweight module that converts pretrained \emph{deep learning} models into pretrained \emph{evidential deep learning} models. ETN requires only a small dataset for training compared to the large datasets used for training the pretrained models.
\item We demonstrate the effectiveness of ETN on both image classification and LLM question-answering (QA) tasks, achieving better uncertainty estimation and lower inference cost than existing post-hoc uncertainty estimation methods, while preserving accuracy.
\end{itemize}

\section{Background}
We briefly review the foundations of \emph{Evidential Deep Learning} (EDL) and recent extensions in this line of work. In addition, since our approach adapts evidential modeling in a post-hoc manner, we also discuss its connection to \emph{post-hoc uncertainty estimation}.

\paragraph{Evidential Deep Learning.}
EDL extends Subjective Logic~\citep{10.5555/3031657}, which expresses subjective opinions through probabilistic logic, to deep learning models. As Subjective Logic provides theorem that there is a \emph{bijectivity between subjective opinions and Dirichlet PDFs}, EDL aims for deep learning models to directly model a Dirichlet distribution over categorical outcomes~\citep{sensoy2018evidentialdeeplearningquantify}. This enables a single forward pass to capture aleatoric and epistemic uncertainty, without requiring sampling or ensembles.

Since its introduction, several extensions have addressed EDL’s limitations. PriorNet~\citep{malinin2018predictive} interprets EDL in a Bayesian framework, decomposing predictive uncertainty into aleatoric, epistemic, and distributional components. 
R-EDL~\citep{chen2024r} identifies the limitation of the overly strict loss formulation in standard EDL and introduces a relaxed objective to improve flexibility. DAEDL~\citep{yoon2024uncertainty} further addresses the inability of standard EDL to reflect the distance between training and test samples, incorporating feature density into the prediction stage for more faithful uncertainty estimates. Finally, IB-EDL~\citep{li2025calibrating} reformulates EDL through the information bottleneck principle to mitigate overconfidence and extends the framework to LLMs.

\paragraph{Post-hoc Uncertainty Estimation.}
To avoid the computational overhead of retraining, post-hoc methods estimate uncertainty from pretrained models~\citep{gawlikowski2023survey}. A common approach is the Laplace approximation, which approximates the model distribution as Gaussian~\citep{daxberger2021laplace}. However, it requires computing the Hessian of model parameters, which limits scalability due to high computational cost. Alternatively, the Dirichlet Meta Model (DMM)~\citep{shen2022posthocuncertaintylearningusing} trains a meta-model that models a Dirichlet prior over softmax outputs using hidden representations from all layers. Despite its effectiveness, DMM relies on access to the original training data and scales with both the depth and dimensionality of the base model. This poses serious challenges for large pretrained models, where the original training data may be unavailable and the resulting DMM can become prohibitively large.




\section{Preliminaries and Problem Statement}
We first introduce the notation and basic setup of EDL, then describe our approach for effectively transforming standard pretrained models into EDL models.
\subsection{Setup and Notation}
\label{sec:problem}
We study multiclass classification on $\mathcal{X}\times\mathcal{Y}$, where $\mathcal{X}$ is the input space and $\mathcal{Y}=\{1,\dots,C\}$ with $C\ge 2$. A classifier $\theta:\mathcal{X}\to\mathcal{Y}$ is composed of a feature extractor $\phi:\mathcal{X}\to\mathbb{R}^{D}$ (with hidden dimension $D$), a classification head $h:\mathbb{R}^{D}\to\mathcal{Y}$.
We treat $\theta$ as a pretrained model trained on a large dataset $\mathcal{D}=\{(x_i,y_i)\}_{i=1}^{N}$. Also, we denote the logit vector generated by the model $\bm z \;=\; (z_1,\dots,z_C)^\top \in \mathbb{R}^C$. 

\subsection{Evidential Deep Learning}
\label{sec:edl}
The key idea of EDL is to build a Dirichlet distribution $\mathrm{Dir}(\bm{\pi} \mid \bm{\alpha})$, where $\bm{\pi} \in \Delta^{C-1}$ represents the categorical probability vector over $C$ classes. Then, the predicted probability is constructed as the expectation of categorical distribution with respect to Dirichlet distribution, i.e., $p=\mathbb{E}_{\mathrm{Dir}(\bm{\pi}\mid \bm{\alpha})}[p(y\mid \bm{\pi})]$. EDL forms Dirichlet parameters $\bm\alpha$ with logit, where $\bm{\alpha}=(\alpha_1,\dots,\alpha_C)^\top$ is induced by $\bm \alpha=f(\bm z)+\bm b$, with $f$ being a monotonically increasing non-negative function, such as \emph{ReLU}, \emph{softplus} and \emph{exponential function}. $\bm b$ is a  prior belief term, which is usually set to $\bm b=\bm 1_{C}$~\citep{sensoy2018evidentialdeeplearningquantify,malinin2018predictive,charpentier2020posterior,deng2023uncertaintyestimationfisherinformationbased}. The overall concentration of the Dirichlet distribution is defined as $\alpha_0 = \sum_{i=1}^{C} \alpha_i$, which captures the model's total evidence or confidence about its prediction.
The optimization of EDL is done through various loss, including sum of squares (MSE)~\citep{sensoy2018evidentialdeeplearningquantify,deng2023uncertaintyestimationfisherinformationbased,chen2024r,li2025calibrating}, Type-2 maximum likelihood~\citep{sensoy2018evidentialdeeplearningquantify} and KL divergence matching~\citep{malinin2018predictive, chen2018variational, charpentier2020posterior, joo2020being}. 
\citet{shen2024uncertainty} provides a unifying view of these framework, where the objective is unified into\footnote{We omit the OOD term from the original formulation as this term require access to an external OOD dataset or a separate density model (e.g., a normalizing flow), introducing additional training complexity and dependencies.}:
\begin{equation}
\begin{aligned}
\mathcal{L}_{\text{EDL}}(\theta) ={\mathbb{E}_{(x,y)\in\mathcal{D}}\!\left[    D_{\text{KL}}\!\left( p^{(\nu)}(\bm{\pi} \mid y),\, p_\theta(\bm{\pi} \mid x) \right)  \right]}
\end{aligned}
\label{eq:unified_objective}
\end{equation}
where $D_{\text{KL}}(\cdot,\cdot)$ is a KL-Divergence of either a forward~\citep{malinin2018predictive,NEURIPS2019_7dd2ae7d} or reverse~\citep{sensoy2018evidentialdeeplearningquantify,joo2020being,NEURIPS2019_7dd2ae7d}. $p^{(\nu)}(\bm{\pi}\mid y)$ is a Dirichlet distribution conditioned on $y$, defined as 
$p^{(\nu)}(\bm{\pi}\mid y)=\mathrm{Dir}(\bm \pi \mid \bm \alpha^y)$ with $\bm \alpha^y= \bm1_C+(\nu-1)\bm{e}_y$, where $\bm{e}_y$ is a one-hot vector placed on the answer label $y$. As for $\nu$, it is known that increasing this value is beneficial~\citep{shen2024uncertainty}, as it can learn to collect evidence that leads to its prediction better for in-distribution (ID) dataset. Meanwhile, $p_\theta(\bm{\pi} \mid x)$ is an Dirichlet distribution formed by the model, defined as $p_\theta(\bm{\pi} \mid x)=\mathrm{Dir}(\bm{\pi}\mid\bm{\alpha})$, with $\bm \alpha=f(\bm z)+\bm b$.

Once the Dirichlet distribution $\mathrm{Dir}(\bm{\pi}\mid\bm{\alpha})$ is formed, uncertainty estimation is done in Bayesian sense. More specifically, by treating the prediction probability $p(y\mid x)$ as marginal probability  $p(y\mid x)=\int p(y\mid x,\bm{\pi})\mathrm{Dir}(\bm{\pi}\mid \bm{\alpha})d\bm{\pi}$, we can utilize following distributional uncertainty estimates to discriminate OOD samples:

\begin{enumerate}
    \item \textbf{Mutual Information:} $I(y;p) \\ \approx \underbrace{H[\mathbb{E}_{\mathrm{Dir}(\bm\pi\mid \bm\alpha)}[p(y\mid \bm\pi)]]}_{\text{Entropy of expectation}}-\underbrace{\mathbb{E}_{\mathrm{Dir}(\bm\pi\mid \bm\alpha)}[H[p(y\mid \bm\pi)]]}_{\text{Expectation of entropy}}$
    \item \textbf{Differential Entropy:} $h[\mathrm{Dir}(\bm{\pi}\mid \bm{\alpha})]$
\end{enumerate}

Along with these distributional uncertainty estimates, max probability $p_\text{max}=\max \{p_1,\dots,p_C\}$~\citep{hendrycks2016baseline} and concentration $\alpha_0$ is also used for total uncertainty estimates~\citep{sensoy2018evidentialdeeplearningquantify}.




\subsection{Post-hoc Evidential Deep Learning} 
\label{sec:post_hoc_edl}
Conventional Bayesian methods for uncertainty estimation are often computationally intractable or prohibitively expensive, especially for large pretrained models.
Adapting EDL for post-hoc uncertainty estimation therefore offers a practical and efficient alternative.

A straightforward solution is to fine-tune the pretrained model with $\mathcal{L}_{\text{EDL}}$. However, since we are highly likely to only have access to the dataset $\mathcal{D'}=\{(x_i,y_i)\}_{i=1}^{N'}$ with much smaller size than that of dataset used for pretraining, i.e, $N'\ll N$, such fine-tuning risks (1) overfitting to $\mathcal{D'}$ and (2) degrading predictive accuracy. 

To avoid these issues, we perform adaptation entirely in the logit space. This choice offers several key advantages: \textbf{(1) Preservation of representations}: Adjusting logits allows us to change the model’s predictive behavior without disturbing its learned feature space. \textbf{(2) Low computational cost}: Logit-space adaptation is lightweight and operates in post-hoc manner, requiring no gradient updates to the backbone. \textbf{(3) Direct connection to uncertainty modeling}: Since EDL defines Dirichlet parameters as functions of logits, adjusting the logits provides a natural way to shape the resulting Dirichlet distribution, controlling uncertainty without retraining the entire network.

Formally, we express this adaptation as an affine transformation of the output logits $\bm{z}$ via a mapping $\bm z'=A\bm z$, where $A$ denotes the transformation parameter, which may be a scalar, vector, or matrix. We denote all such cases uniformly by $A$ for generality. The transformed logits are then used to form evidential parameters $\bm\alpha'=f(\bm z')+\bm b$, yielding a Dirichlet distribution $\mathrm{Dir}(\bm{\pi}\mid\bm{\alpha}')$. The transformed predictive probability is then computed as $p'=\mathbb{E}_{\mathrm{Dir}(\bm{\pi}\mid\bm{\alpha}')}[p(y\mid \bm{\pi})]$.

\subsection{Sample-Dependent Transformation Parameter}
We now turn to learning the optimal transformation parameter \(A\). To understand how \(A\) should be formed, we compare models trained with EDL and those trained with cross-entropy loss \(\mathcal{L}_{\mathrm{CE}}:\mathbb{R}^C\times\mathcal{Y}\to\mathbb{R}\), which is used for most pretrained models (e.g. ImageNet classifier, LLMs). We analyze the difference from two perspectives: (1) the \emph{logit-space} behavior, and (2) a \emph{Bayesian} interpretation.

\paragraph{Logit-space view.} 
In EDL training, the total concentration $\alpha_0$ is explicitly regulated to reflects the model's confidence, as it becomes large for confident ID samples and small for uncertain or OOD inputs~\citep{shen2024uncertainty,juergens2024epistemic}. Since $\bm\alpha$ is derived from the logits through a monotonic mapping $f$, the logit magnitude directly controls the model’s estimated uncertainty.
However, cross-entropy training provides no such constraint, as it minimizes loss without explicitly regulating the scale of logits or their implied uncertainty.
We formalize this distinction as:

\begin{proposition}[Cross-entropy does not identify Dirichlet concentration]
\label{prop:alpha}
Assume the training data are separable and the model has infinite capacity, so that
logits $\bm z$ can be set arbitrarily. Then:
\begin{itemize}
    \item There exists logit $\tilde{\bm z}$ such that
    $\mathcal{L}_{\rm CE}(\tilde{\bm z},y)\to 0$
    while $\tilde\alpha_0<\infty$.
    \item There exists logit $\hat{\bm z}$ such that
    $\mathcal{L}_{\rm CE}(\hat{\bm z},y)\to 0$
    and $\hat\alpha_0\to\infty$.
\end{itemize}
Hence, minimizing cross-entropy alone does not determine the total concentration $\alpha_0$.
\end{proposition}

See the Supplementary Material for the proof. Proposition~\ref{prop:alpha} implies that a cross-entropy–trained model can exhibit arbitrary \(\alpha_0\) across samples since the loss provides no constraint on its scale. This motivates a need for \emph{sample-dependent} transformation parameter \(A\).

\paragraph{Bayesian view.} EDL can be interpreted as constructing a per-sample posterior Dirichlet distribution~\citep{yoon2024uncertainty}, given by
\begin{equation}
\label{eq:bayesian}
    p(\bm \pi \mid x;\bm\alpha_{\text{post}}) \propto  
    p_\theta(x \mid \bm \pi;\bm\alpha_x)\, p(\bm \pi;\bm\alpha_{\text{prior}})
\end{equation}
where the posterior is $\mathrm{Dir}(\bm\pi\mid\bm\alpha_{\text{post}})$ with $\bm\alpha_{\text{post}}=\bm\alpha_x+\bm\alpha_{\text{prior}}$. Here, $\bm\alpha_x$ (identical to $f(\bm z)$) corresponds to the evidence predicted from the sample, and $\bm\alpha_{\text{prior}}$ (identical to $\bm b$) specifies the prior belief. 

This Bayesian view highlights that EDL models a per-sample \emph{distribution over categorical probabilities}, whereas cross-entropy produces only a single \emph{categorical probability vector} for each sample.
Consequently, $A$ must be \emph{sample-dependent} to capture these per-sample posterior variations.



\section{Evidential Transformation Network}
\subsection{Effective Adaptive Transformation Strategy}
Since the transformation parameter $A$ must depend on each input sample, a natural starting point would be \emph{Adaptive Temperature Scaling} (AdaTS)~\citep{joy2023sample}. AdaTS constructs sample-dependent temperatures using a Gaussian mixture prior, where each Gaussian corresponds to a class.
This prior is learned via a Variational Autoencoder (VAE) and then used by a temperature prediction network to output a single deterministic scaling value.

Although effective, we opt to model a distribution over the $A$ and treat it as the direct prior of the  transformed predictive probability $p'$. This formulation allows $p'$ to be expressed as a marginal over $A$, enabling the model to learn a full variational distribution rather than a deterministic value. This approach allows the transformation to operate within a probabilistic framework, offering greater expressiveness and flexibility in modeling uncertainty.
\subsection{Training Objective}
\paragraph{Optimizing with ELBO.}
Our objective is to optimize the transformed predictive distribution such that it minimizes the EDL loss defined in Equation~\ref{eq:unified_objective}. 

\noindent Starting from the forward KL formulation,
\begin{equation}
\begin{aligned}
\mathcal{L_\text{EDL}}
&=\mathbb{E}_{(x,y)\in\mathcal{D'}}\!\left[
D_{\mathrm{KL}}\!\big(p^{(\nu)}(\bm\pi\mid y)\,\|\,p'(\bm\pi\mid x;\theta)\big)
\right]\\
&=\underbrace{\mathbb{E}_{\mathcal{D'}}\,\mathbb{E}_{p^{(\nu)}(\bm\pi\mid y)}\big[\log p^{(\nu)}(\bm\pi\mid y)\big]}_{\text{const.\ w.r.t.\ }\theta}\\
&\hspace{1em}-\mathbb{E}_{\mathcal{D'}}\,\mathbb{E}_{p^{(\nu)}(\bm\pi\mid y)}\big[\log p'(\bm\pi\mid x;\theta)\big] 
\end{aligned}
\label{eq:etn_loss}
\end{equation}


As the first term is constant with respect to $\theta$, we only consider the second term. 
To make this tractable, we introduce an Evidence Lower Bound (ELBO) to approximate $\log p'(\bm\pi\mid x;\theta)$. Specifically, we define a variational distribution $q_{\theta_{\text{ETN}}}(A\mid x)$ to approximate the true posterior $q_{\theta_{\text{ETN}}}(A\mid \bm \pi, x)$, together with a prior $p(A)$:

\begin{equation}
\begin{aligned}
\log p'(\bm\pi\mid x;\theta) &\ge \mathbb{E}_{q_{\theta_{\text{ETN}}}(A\mid x)}\!\big[\log p'(\bm\pi\mid A,x;\theta)\big]\\
&\hspace{1em}- D_{\mathrm{KL}}\!\big(q_{\theta_{\text{ETN}}}(A\mid  x)\,\|\,p(A)\big)
\end{aligned}
\label{eq:lowerbound}
\end{equation}
Substituting Equation~\ref{eq:lowerbound} into \ref{eq:etn_loss} yields the upper bound on  $\mathcal{L}_\text{EDL}$ and the training loss to minimize:

{\footnotesize
\begin{equation}
\begin{aligned}
\mathcal{L}_{\text{ETN}}(\theta_{\text{ETN}})
&= -\,\mathbb{E}_{\mathcal{D'}}\,
\mathbb{E}_{p^{(\nu)}(\bm\pi\mid y)}\,
\mathbb{E}_{q_{\theta_{\text{ETN}}}(A\mid x)}
\big[\log p'(\bm\pi\mid A,x;\theta)\big] \\
&\quad +\;\lambda\,
\mathbb{E}_{x\sim\mathcal{D'}}\Big[
D_{\mathrm{KL}}\!\big(q_{\theta_{\text{ETN}}}(A\mid x)\,\|\,p(A)\big)\Big] \nonumber
\end{aligned}
\label{eq:etn_loss_final}
\end{equation}
}
\noindent where $\lambda$ is a regularization coefficient that balances the reconstruction and KL terms. With a change of formula, let $p'(\bm\pi\mid A,x;\theta)=\mathrm{Dir}(\bm\pi\mid \bm\alpha')$ with
$\bm\alpha' = f(A\bm z)+\bm b$ and $\alpha'_0=\sum_{i=1}^C \alpha'_i$.
Then
\begin{equation}
\begin{aligned}
\log p'(\bm\pi\mid A,x;\theta)
= \log\frac{\Gamma(\alpha'_0)}{\prod_{i=1}^C\Gamma(\alpha'_i)}
  + \sum_{i=1}^C (\alpha'_i-1)\log \pi_i \nonumber
\end{aligned}
\end{equation}

\noindent Similarly, let $p^{(\nu)}(\bm\pi\mid y)=\mathrm{Dir}(\bm\pi\mid \bm\alpha^y)$ with
$\bm\alpha^y=\bm 1_C+(\nu-1)\,\bm e_y$.
Using $\mathbb{E}_{\mathrm{Dir}(\bm\alpha^y)}[\log \pi_i]=\psi(\alpha^y_i)-\psi(\alpha^y_0)$,
along with Monte-Carlo approximation, we draw $A^{(m)}\sim q_{\theta_{\text{ETN}}}(A\mid x)$, for $m=1,\dots,M$ leading to the final per-sample loss for $(x,y)\in\mathcal{D'}$:
\begin{equation}
\begin{aligned}
\widehat{\mathcal{L}}_{\text{ETN}}(\theta_\text{ETN})
&= -\,\frac{1}{M}\sum_{m=1}^M
\Big[\log\frac{\Gamma(\alpha'^{(m)}_0)}{\prod_{i=1}^C \Gamma(\alpha'^{(m)}_i)}
\\&\quad + \sum_{i=1}^C (\alpha'^{(m)}_i-1)\big(\psi(\alpha^y_i)-\psi(\alpha^y_0)\big)\Big]\\
&\quad +\;\lambda\, D_{\mathrm{KL}}\!\big(q_{\theta_{\text{ETN}}}(A\mid x)\,\|\,p(A)\big)
\end{aligned}
\label{eq:final_loss}
\end{equation}
where $\bm\alpha'^{(m)} = f(A^{(m)}\bm z)+\bm b$.

\paragraph{Training and Inference.}
The Evidential Transformation Network, denoted \(\theta_{\text{ETN}}\), consists of MLP layers that take as input a hidden representation from the base model \(\theta\) and model a variational distribution over the transformation parameters \(q_{\theta_{\text{ETN}}}(A \mid x)\).
Following prior work~\citep{joy2023sample,10.5555/3692070.3693889}, we use the last hidden state of \(\theta\) as the input to \(\theta_{\text{ETN}}\).

In addition, we set the prior term \(\bm b\) as a learnable parameter. This design has two advantages: (i) \emph{Subjective Logic view.} In Subjective Logic, $\bm b$ is originally derived with the multiplication of two factors, base rate \(\bm a\) and prior strength of the opinion \(W\). Many EDL works fix \(a_i=\tfrac{1}{C}\) and \(W=C\), yielding \(\bm b=\bm 1_C\). However, \citet{chen2024r} shows that this hard prior can be suboptimal, and relaxing the prior and setting \(\bm b\) as a hyperparameter leads to improvement. (ii) \emph{Bayesian view.} In the decomposition \(\boldsymbol\alpha_{\text{post}}=\boldsymbol\alpha_x+\boldsymbol\alpha_{\text{prior}}\) (Equation~\ref{eq:bayesian}), fixing \(\bm b=\bm 1_C\) imposes a strong prior. \citet{yoon2024uncertainty} demonstrate that completely ignoring the prior by setting them to zero vector, i.e, \(\bm b=\bm 0_C\) improves evidential modeling. 

We train \(\theta_{\text{ETN}}\) and \(\bm b\) using the objective in Equation~\ref{eq:final_loss}. At inference, we transform original predictive probabilities by marginalizing over \(A\) using a Monte-Carlo approximation:
\begin{align}
    p'=\frac{1}{M}\sum_{m=1}^{M}\mathbb{E}_{\mathrm{Dir}(\bm \pi \mid \bm \alpha'^{(m)})}[p(y\mid \bm \pi)] \label{eq:dirichlet_expectation}
\end{align}
where \(\bm \alpha'^{(m)}=f(A^{(m)}\bm z)+\bm b\) and \(A^{(m)}\sim q_{\theta_{\text{ETN}}}(A\mid x)\). For the training and inference algorithms, as well as the architectural details of \(\theta_{\text{ETN}}\), please refer to the Supplementary Material.

\begin{figure*}[ht]
    \centering
    \begin{minipage}{0.5\textwidth}
        \centering
        \includegraphics[width=\linewidth]{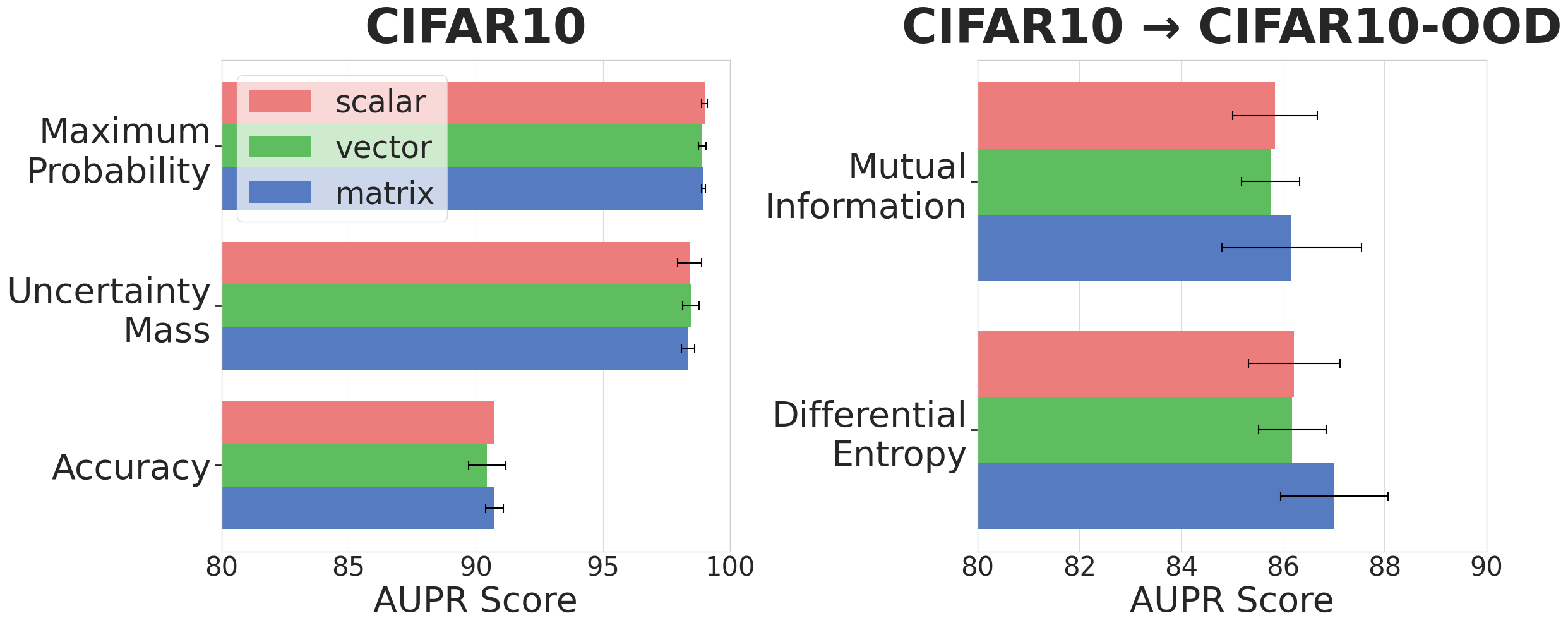}
    \end{minipage}\hfill
    \begin{minipage}{0.5\textwidth}
        \centering
        \includegraphics[width=\linewidth]{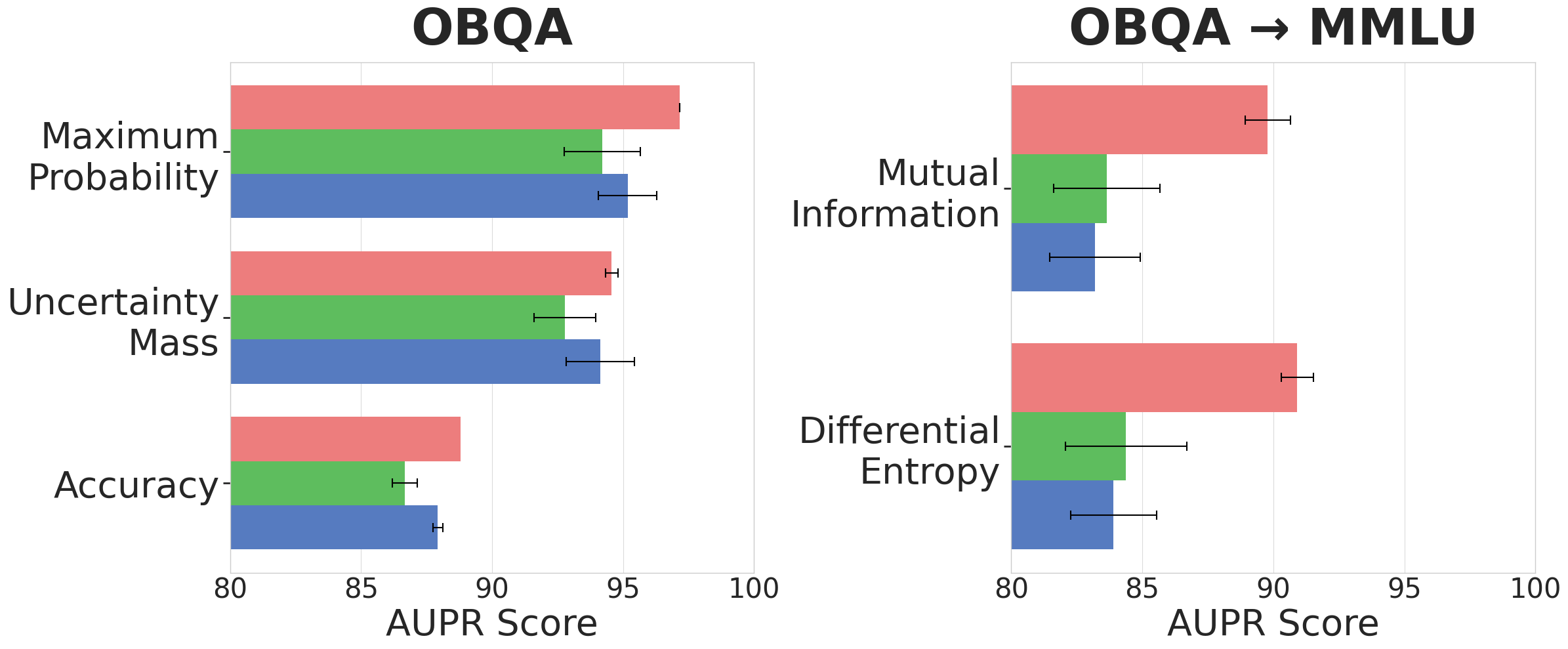}
    \end{minipage}\hfill
    \caption{Comparison of uncertainty estimation performance based on different dimension of transformation parameter $A$.}
    \label{fig:dim}
\end{figure*}

\subsection{Additional Details}\label{sec:add_details}

\paragraph{Using softplus as $f$.} Choosing an appropriate $f$ is crucial. Early EDL works commonly used ReLU, but a plain ReLU is sub-optimal when the predicted evidence collapses to zero (i.e., $f(\bm z)=\bm 0_C$)~\citep{pandey2023learn}. Alternatives such as \emph{softplus}~\citep{deng2023uncertaintyestimationfisherinformationbased} or the \emph{exponential function}~\citep{yoon2024uncertainty} are able to avoid this problem by ensuring gradient to flow across all samples.

In our setting, one might consider the \emph{exponential function} a natural choice for $f$, since the pretrained classifier models probabilities via softmax. 
However, this choice can lead to severe numerical instability. 
As shown in Proposition~\ref{prop:alpha}, some pretrained models may already produce very large logit values.
Applying the \emph{exponential} then amplifies these values, causing the total concentration $\alpha_0$ to grow exponentially. 
Unlike cross-entropy training which avoids explicitly computing $\sum_i^Ce^{z_i}$ by relying on the log-sum-exp trick, no analogous stabilization exists for the EDL objective, where $\alpha_0$ must be computed directly. 
As a result, using the exponential can drive $\alpha_0$ to astronomically large values during initialization or training, leading to numerical overflow and unstable gradients.

We therefore adopt \emph{softplus} as $f$. \emph{softplus} guarantees positivity (preventing the zero-evidence dead zone) yet grows only linearly for large positive inputs and exponentially only for large negative inputs, which naturally relaxes $\alpha_0$ and ensures more stable optimization.

\paragraph{Modeling the Variational Distribution.}  
To preserve the predictive capability of the pretrained model, we restrict the variational distribution to families that enforce monotonic rescaling of logits.
Specifically, we adopt the \emph{Gamma distribution}, whose support lies in the positive real domain.
Detailed modeling choices for different transformation dimensionalities are provided in the Supplementary Material.

\paragraph{Choice of Prior.}
For the prior over the $A$, we again work in the \emph{logit space} to compare models trained with EDL and cross-entropy.
In particular, we view the scale of logits from a \emph{margin} perspective~\citep{suykens1999least,liu2016large}, which focuses on maximizing the separation between class representations.
We formally define the inter-class margin for a sample $(x, y)$ as
\[\gamma(\bm z, y) = z_y - \max_{j \ne y} z_j.\]
With this definition, we can establish a connection between the margins induced by cross-entropy and EDL training.

\begin{theorem}[EDL vs.\ CE margin under equal loss]
\label{thm:equal-ce-margin}
Assume that a CE-trained model and an EDL-trained model yield the same per-sample loss 
\(L := \mathcal{L}_{\mathrm{CE}}(\bm z, y) = \mathcal{L}_{\mathrm{EDL}}(\bm z, y)\).  
Further, assume that for the EDL model there exists \(\eta\) with \(0 \le \eta < \nu - b_y\) such that
\[
\alpha_y \ge \nu - \eta,
\qquad
\alpha_j \le b_j + \eta \quad \forall j \ne y.
\]
Then the probability that the EDL margin exceeds the CE margin under equal loss is bounded by
\begin{equation}
\label{eq:edl_vs_ce_margin}
\begin{aligned}
&P\big(\gamma_{\mathrm{EDL}}(\bm z, y)\ \ge\ \gamma_{\mathrm{CE}}(\bm z, y)\big) \\
&\quad \ge\
P\!\Big(
L
\ge 
\log\!\Big(1 + \tfrac{C - 1}{e^{f^{-1}(\nu - b_y - \eta) - f^{-1}(\eta)}}\Big)
\Big).
\end{aligned}
\end{equation}
\end{theorem}

\begin{corollary}[EDL model using softplus as $f$]
\label{cor:softplus-equal-ce}
Under the conditions of Theorem~\ref{thm:equal-ce-margin}, let $f$ be the softplus function.  
Then the probability that the EDL margin exceeds the CE margin at the same loss is bounded by
\begin{equation}
\label{eq:softplus-edl-lb-equal-ce}
\begin{aligned}
&P\big(\gamma_{\mathrm{EDL}}(\bm z,y)\ \ge\ \gamma_{\mathrm{CE}}(\bm z,y)\big) \\
&\quad \ge\
P\!\Big(
L
\ge 
\log\big(1 + (C - 1)\tfrac{e^{\eta} - 1}{e^{\nu - b_y - \eta} - 1}\big)
\Big).
\end{aligned}
\end{equation}
\end{corollary}

Proofs for Theorem~\ref{thm:equal-ce-margin} and Corollary~\ref{cor:softplus-equal-ce} are provided in the Supplementary Material. Since EDL models are trained with large $\nu$ ($\sim10^4$), and well-trained models generally satisfy $\nu \gg \eta$, the condition in Corollary~\ref{cor:softplus-equal-ce} typically holds.
This indicates that EDL training with \emph{softplus} as $f$ 
 is likely to yield larger inter-class margins than cross-entropy training for the same loss.

To enlarge the margin during transformation, we scale the logits by choosing a prior on $A$ whose mode is greater than 1. In addition, to avoid imposing an overly strong prior, we fix the variance to $\mathrm{Var}(A) = 5$.


\section{Experiments}
We begin by comparing the performance of transformation parameters with varying dimensionalities.
We then present the main results on image classification and LLM question-answering tasks.
Finally, we evaluate alternative transformation strategies to assess whether ETN offers an effective  approach to post-hoc EDL transformation.

\subsection{Experimental Setting}

\paragraph{Image Classification Datasets.}
We evaluate on CIFAR-10~\citep{cifar10} and ImageNet~\citep{5206848}. For OOD detection, we use SVHN~\citep{37648} and CIFAR-100~\citep{cifar10} as OOD sets for CIFAR-10, and ImageNet-A~\citep{Hendrycks_2021_CVPR}, ImageNet-Sketch~\citep{wang2019learning}, and ImageNet-R~\citep{hendrycks2021many} as OOD sets for ImageNet. 

Since our focus is post-hoc uncertainty estimation, we use 5\% of the CIFAR-10 training data for adaptation for both ETN and all baselines, while the remaining 95\% is used to pretrain the model with cross-entropy loss. For ImageNet, as the test set is not publicly available, we split the validation set into 20\% for adaptation and 80\% for evaluation.

\paragraph{LLM Datasets.}
We evaluate on two multiple-choice QA benchmarks with a fixed number of answer choices: OBQA~\citep{OpenBookQA2018} and RACE~\citep{lai-etal-2017-race}.
For OOD evaluation, we use the same dataset for both tasks, consisting of three subsets from MMLU~\citep{hendryckstest2021}: \emph{mathematics}, \emph{computer science}, and \emph{health}.

For both OBQA and RACE, to ensure they represent in-distribution data, we train the LLMs on their respective training sets using cross-entropy loss and use the validation sets to adapt all methods for post-hoc uncertainty estimation.

\paragraph{Models.} Following prior work on EDL~\citep{shen2022posthocuncertaintylearningusing,deng2023uncertaintyestimationfisherinformationbased,chen2024r}, we use VGG16 for CIFAR-10. For ImageNet, where training from scratch is computationally expensive, we use a ResNet50 checkpoint pretrained on \texttt{IMAGENET1K\_V2} provided by \texttt{TorchVision}. For both OBQA and RACE, we use Llama-3.1-8B~\citep{grattafiori2024llama}.

Since LLMs have large vocabulary spaces, which can confound the interpretation of the results, we restrict the output space to only the multiple-choice answer tokens—\texttt{A}, \texttt{B}, \texttt{C}, and \texttt{D}—following prior work~\cite{yangbayesian,li2025calibrating}.

\paragraph{Baselines.}\label{para:image_baselines} 
We consider two groups of baselines:  
(1) \textit{Conventional uncertainty estimation methods}: fine-tuning the full model with cross-entropy loss \textbf{$(\text{MAP}_{\text{CE}})$}, Deep Ensemble \textbf{(DeepEns)}~\citep{NIPS2017_9ef2ed4b}, MC-Dropout \textbf{(MCD)}~\citep{pmlr-v48-gal16}, and Laplace approximation \textbf{(LA)}~\citep{daxberger2021laplace}. For LLMs, instead of laplace approximation, we use Laplace-LoRA \textbf{(LL)}~\citep{yangbayesian}.  
(2) \textit{Post-hoc EDL methods}: fine-tuning with EDL loss \textbf{$(\text{MAP}_{\text{EDL}})$}, EDL through information-bottleneck (\textbf{IB-EDL})~\citep{li2025calibrating}, and the Dirichlet Meta-Model \textbf{(DMM)}~\citep{shen2022posthocuncertaintylearningusing}. 

To compare different transformation strategies, we additionally evaluate static scaling~\citep{guo2017calibration} and AdaTS~\citep{joy2023sample}. For fair comparison to ETN, we set $\bm b$ as a trainable parameter in all EDL-based approaches.

\begin{table*}[t]
\centering
\setlength{\tabcolsep}{5pt}

\begin{subtable}[t]{\textwidth}
\centering
\resizebox{0.8\linewidth}{!}{
\begin{tabular}{lccc ccc ccc} 
\toprule
\multirow{2}{*}{Method}
  & \multicolumn{3}{c}{CIFAR-10}
  & \multicolumn{3}{c}{CIFAR-10 $\to$ SVHN}
  & \multicolumn{3}{c}{CIFAR-10 $\to$ CIFAR-100} \\
\cmidrule(lr){2-4} \cmidrule(lr){5-7} \cmidrule(lr){8-10}
& ACC & MP & UM
& MP & MI & DE
& MP & MI & DE \\
\midrule
$\mathrm{MAP}_{\text{CE}}$ & \result{87.95}{1.2} & \result{97.66}{0.35} & -- 
& \result{73.26}{5.59} & -- & -- 
& \result{79.10}{2.29} & -- & -- \\
DeepEns & \result{90.69}{0.56} & \result{98.93}{0.11} & -- 
& \result{83.71}{1.63} & \result{26.73}{0.09} & -- 
& \result{85.76}{0.77} & \result{48.66}{0.24} & -- \\
MCD & \result{87.12}{1.05} & \result{95.68}{1.41} & -- 
& \result{68.38}{3.54} & \result{72.44}{1.55} & -- 
& \result{74.14}{4.88} & \result{78.10}{2.82} & -- \\
LA & \result{\underline{89.05}}{0.15} & \result{\underline{98.66}}{0.03} & -- 
& \result{77.21}{0.11} & \result{75.69}{0.13} & -- 
& \result{\underline{84.97}}{0.11} & \result{84.59}{0.10} & -- \\
$\mathrm{MAP}_{\text{EDL}}$ & \result{86.81}{1.06} & \result{97.94}{0.23} & \result{\underline{97.81}}{0.18} 
& \result{75.88}{2.37} & \result{\underline{76.42}}{3.81} & \result{76.22}{3.01} 
& \result{83.75}{0.33} & \result{\underline{84.62}}{0.28} & \result{84.43}{0.23} \\
DMM & \result{87.10}{3.20} & \result{98.46}{0.24} & \result{96.83}{0.77} 
& \result{\underline{81.31}}{2.54} & \result{75.43}{11.54} & \result{\underline{81.59}}{6.95} 
& \result{82.65}{1.92} & \result{81.61}{4.35} & \result{\underline{84.61}}{2.54} \\
IB-EDL & \result{88.69}{1.19} & \result{97.90}{0.58} & \result{97.75}{0.51} 
& \result{62.37}{5.54} & \result{61.72}{5.67} & \result{62.16}{5.62} 
& \result{78.96}{2.83} & \result{78.91}{2.51} & \result{79.08}{2.75} \\
\rowcolor{gray!15}
ETN & \result{\textbf{90.70}}{0.00} & \result{\textbf{98.99}}{0.11} & \result{\textbf{98.41}}{0.46} 
& \result{\textbf{85.19}}{1.55} & \result{\textbf{85.22}}{1.05} & \result{\textbf{85.60}}{1.48} 
& \result{\textbf{86.67}}{0.28} & \result{\textbf{86.47}}{0.62} & \result{\textbf{86.84}}{0.32} \\
\bottomrule
\end{tabular}}
\end{subtable}

\vspace{0.75em}
\begin{subtable}[t]{\textwidth}
\centering
\resizebox{0.9\linewidth}{!}{
\begin{tabular}{lccc ccc ccc ccc}
\toprule
\multirow{2}{*}{Method}
  & \multicolumn{3}{c}{ImageNet}
  & \multicolumn{3}{c}{ImageNet $\to$ ImageNet-A}
  & \multicolumn{3}{c}{ImageNet $\to$ ImageNet-S}
  & \multicolumn{3}{c}{ImageNet $\to$ ImageNet-R} \\
\cmidrule(lr){2-4} \cmidrule(lr){5-7} \cmidrule(lr){8-10} \cmidrule(lr){11-13}
& ACC & MP & UM
& MP & MI & DE
& MP & MI & DE
& MP & MI & DE \\
\midrule
$\mathrm{MAP}_{\text{CE}}$
  & \result{46.03}{0.3} & \result{\underline{79.01}}{0.5} & --
  & \result{92.50}{0.0} & -- & --
  & \result{64.19}{2.5} & -- & --
  & \result{73.09}{0.8} & -- & -- \\

DeepEns & \result{26.98}{1.0} & \result{38.55}{1.6} & -- 
& \result{85.02}{1.0} & \result{84.16}{0.1} & -- 
& \result{31.73}{0.9} & \result{42.67}{0.1} & --
& \result{58.90}{3.2} & \result{42.41}{0.4} & --\\

MCD
  & \result{41.14}{1.4} & \result{75.28}{1.3} & --
  & \result{91.75}{0.5} & \result{83.79}{0.1} & --
  & \result{58.90}{3.2} & \result{42.41}{0.4} & --
  & \result{72.82}{0.9} & \result{56.79}{0.6} & -- \\

LA
  & \result{\underline{69.81}}{0.0} & \result{0.13}{0.0} & --
  & \result{\underline{95.07}}{0.0} & \result{\textbf{96.36}}{0.0} & --
  & \result{\underline{71.35}}{0.0} & \result{\textbf{78.25}}{0.0} & --
  & \result{\underline{79.96}}{0.1} & \result{\textbf{84.05}}{0.0} & -- \\

$\mathrm{MAP}_{\text{EDL}}$
  & \result{32.97}{0.3} & \result{56.99}{0.6} & \result{\underline{35.09}}{0.4}
  & \result{92.42}{0.3} & \result{90.47}{0.4} & \result{\underline{90.63}}{0.5}
  & \result{64.74}{4.6} & \result{65.74}{4.3} & \result{\underline{66.06}}{4.4}
  & \result{73.51}{2.0} & \result{76.50}{2.0} & \result{\underline{76.63}}{2.0} \\

DMM
  & \result{2.53}{0.3} & \result{20.33}{0.2} & \result{18.22}{0.3}
  & \result{83.78}{0.1} & \result{83.77}{0.1} & \result{83.76}{0.1}
  & \result{42.44}{0.6} & \result{42.50}{0.5} & \result{42.47}{0.5}
  & \result{56.51}{0.7} & \result{56.58}{0.6} & \result{56.57}{0.6} \\

IB-EDL
  & \result{23.47}{2.5} & \result{70.00}{2.0} & \result{9.60}{1.3}
  & \result{87.73}{0.9} & \result{84.28}{0.1} & \result{83.74}{0.1}
  & \result{51.81}{2.4} & \result{43.31}{0.6} & \result{43.92}{0.1}
  & \result{64.42}{1.7} & \result{56.90}{0.2} & \result{56.74}{0.1} \\
\rowcolor{gray!15}
ETN
  & \result{\textbf{79.61}}{0.0} & \result{\textbf{88.04}}{0.1} & \result{\textbf{85.29}}{0.1}
  & \result{\textbf{95.78}}{0.1} & \result{\underline{93.98}}{0.1} & \result{\textbf{93.60}}{0.3}
  & \result{\textbf{74.91}}{0.5} & \result{\underline{68.03}}{1.0} & \result{\textbf{67.26}}{1.1}
  & \result{\textbf{83.65}}{0.5} & \result{\underline{79.40}}{0.8} & \result{\textbf{78.65}}{1.0} \\
\bottomrule
\end{tabular}}
\end{subtable}
\caption{
AUPR scores on CIFAR-10, SVHN, and CIFAR-100 (top), and on ImageNet, ImageNet-A, ImageNet-S, and ImageNet-R (bottom).
}
\label{tab:main2}
\end{table*}

\begin{table*}[t]
\centering
\setlength{\tabcolsep}{5pt}

\resizebox{0.9\linewidth}{!}{
\begin{tabular}{lccc ccc ccc ccc}
\toprule
\multirow{2}{*}{Method}
  & \multicolumn{3}{c}{RACE}
  & \multicolumn{3}{c}{RACE $\to$ MMLU}
  & \multicolumn{3}{c}{OBQA}
  & \multicolumn{3}{c}{OBQA $\to$ MMLU} \\
\cmidrule(lr){2-4} \cmidrule(lr){5-7} \cmidrule(lr){8-10} \cmidrule(lr){11-13}
& ACC & MP & UM
& MP & MI & DE
& ACC & MP & UM
& MP & MI & DE \\
\midrule
$\mathrm{MAP}_{\text{CE}}$
  & \result{89.54}{0.1} & \result{\underline{97.05}}{0.1} & --
  & \result{96.01}{0.3} & -- & -- 
  & \result{87.07}{0.2} & \result{96.94}{0.2} & --
  & \result{\underline{89.83}}{0.9} & -- & -- \\

DeepEns
  & \result{85.84}{0.1} & \result{96.66}{0.2} & --  
  & \result{94.28}{0.5} & \result{91.42}{1.1} & --
  & \result{83.47}{0.4} & \result{94.51}{0.2} & --
  & \result{87.08}{1.0} & \result{\underline{83.68}}{0.8} & --  \\
  
MCD
  & \result{\underline{89.66}}{0.0} & \result{96.94}{0.0} & --
  & \result{\underline{96.59}}{0.0} & \result{\textbf{96.33}}{0.0} & -- 
  & \result{86.40}{0.2} & \result{\textbf{97.18}}{0.0} & -- 
  & \result{88.99}{0.0} & \result{82.38}{0.0} & -- \\

LL
  & \result{88.74}{0.2} & \result{64.98}{0.3} & --
  & \result{25.24}{0.9} & \result{27.22}{0.8} & -- 
  & \result{86.07}{0.1} & \result{51.27}{0.2} & --  
  & \result{24.74}{0.6} & \result{24.63}{0.3} & -- \\
  
$\mathrm{MAP}_{\text{EDL}}$
  & \result{84.94}{0.2} & \result{91.70}{0.4} & \result{87.86}{1.5} 
  & \result{91.19}{0.3} & \result{87.67}{1.0} & \result{91.18}{0.3}
  & \result{80.27}{1.2} & \result{85.62}{1.8} & \result{76.69}{1.7} 
  & \result{79.86}{1.9} & \result{70.16}{1.8} & \result{79.95}{1.9} \\

DMM
  & \result{89.07}{0.2} & \result{96.76}{0.0} & \result{\underline{95.20}}{0.5}
  & \result{93.11}{0.8} & \result{92.06}{0.6} & \result{93.25}{0.7}
  & \result{\underline{87.20}}{0.1} & \result{95.30}{0.2} & \result{\underline{93.46}}{0.8} 
  & \result{85.11}{0.7} & \result{81.59}{1.8} & \result{84.78}{0.6} \\

IB-EDL
  & \result{86.00}{0.7} & \result{96.20}{0.5} & \result{92.70}{0.4}
  & \result{94.45}{0.7} & \result{85.95}{1.4} & \result{\underline{94.38}}{0.7}
  & \result{81.60}{0.0} & \result{94.33}{0.2} & \result{83.10}{0.6} 
  & \result{89.74}{1.0} & \result{68.20}{1.8} & \result{\underline{89.87}}{1.0} \\
\rowcolor{gray!15}
ETN
  & \result{\textbf{89.69}}{0.0} & \result{\textbf{97.60}}{0.0} & \result{\textbf{96.00}}{0.2} 
  & \result{\textbf{96.80}}{0.0} & \result{\underline{94.51}}{1.3} & \result{\textbf{94.65}}{1.3}
  & \result{\textbf{88.80}}{0.0} & \result{\underline{97.15}}{0.0} & \result{\textbf{94.57}}{0.2} 
  & \result{\textbf{91.70}}{0.0} & \result{\textbf{89.79}}{0.9} & \result{\textbf{90.91}}{0.6} \\
\bottomrule
\end{tabular}
}
\caption{AUPR scores on RACE and OBQA with MMLU subsets as OOD.}
\label{tab:main3}
\end{table*}
\paragraph{Uncertainty Metrics.} 
For confidence estimation on in-distribution (ID) datasets, we use \textbf{Maximum Probability (MP)} and \textbf{Uncertainty Mass (UM)}, where UM corresponds to the total concentration $\alpha_0$. We measure the Area under the Precision–Recall Curve (AUPR) between prediction correctness (1 for correct and 0 for incorrect) and these metrics.
For OOD detection, we evaluate \textbf{MP} to capture total uncertainty, and \textbf{Mutual Information (MI)} and \textbf{Differential Entropy (DE)} to capture distributional uncertainty.
AUPR is computed by treating ID samples as positive (label 1) and OOD samples as negative (label 0).
When multiple OOD datasets are used, we report the average score across them and denote it as \{\texttt{ID\_dataset\_name}\}-OOD. 


\subsection{Comparison Across Transformation Dimensionalities}
\label{sec:compare_dim}
To examine how the dimensionality of the transformation parameter affects model behavior, we compare scalar, vector, and matrix transformations. Note that the scalar variant \emph{preserves the logit ordering} and therefore \emph{retains accuracy}. The results are summarized in Figure~\ref{fig:dim}. Overall, the outcomes are mixed. On CIFAR-10, all three configurations perform comparably in ID confidence estimation, with the matrix variant achieving the best OOD detection.
However, on OBQA, the scalar variant outperforms both the vector and matrix variants in both ID and OOD settings. We also note that vector and matrix transformations frequently lead to reductions in predictive accuracy across experiments.

Why, then, do higher-dimensional transformations not consistently outperform the scalar version?
We attribute this to two main factors. First, under common EDL objectives, the primary effect is to control the Dirichlet concentration~\cite{shen2024uncertainty}, for which scalar scaling is sufficient to shape the Dirichlet parameters. Second, because pretrained models already possess strong predictive capability, higher-dimensional transformations introduce unnecessary flexibility, increasing the risk of overfitting to the validation data.

Therefore, we adopt the scalar setting as the default in all subsequent experiments.

\begin{figure*}[ht]
    \centering
    \begin{minipage}{0.5\textwidth}
        \centering
        \includegraphics[width=\linewidth]{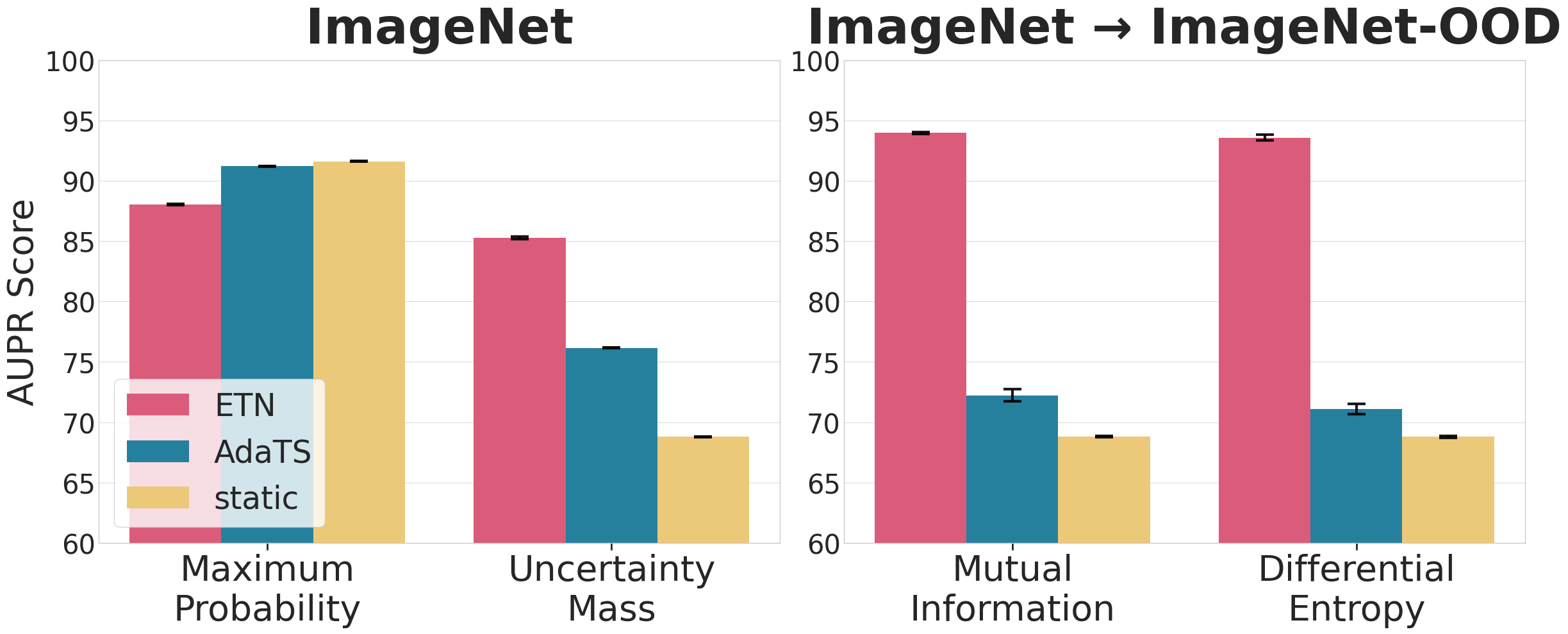}
    \end{minipage}\hfill
    \begin{minipage}{0.5\textwidth}
        \centering
        \includegraphics[width=\linewidth]{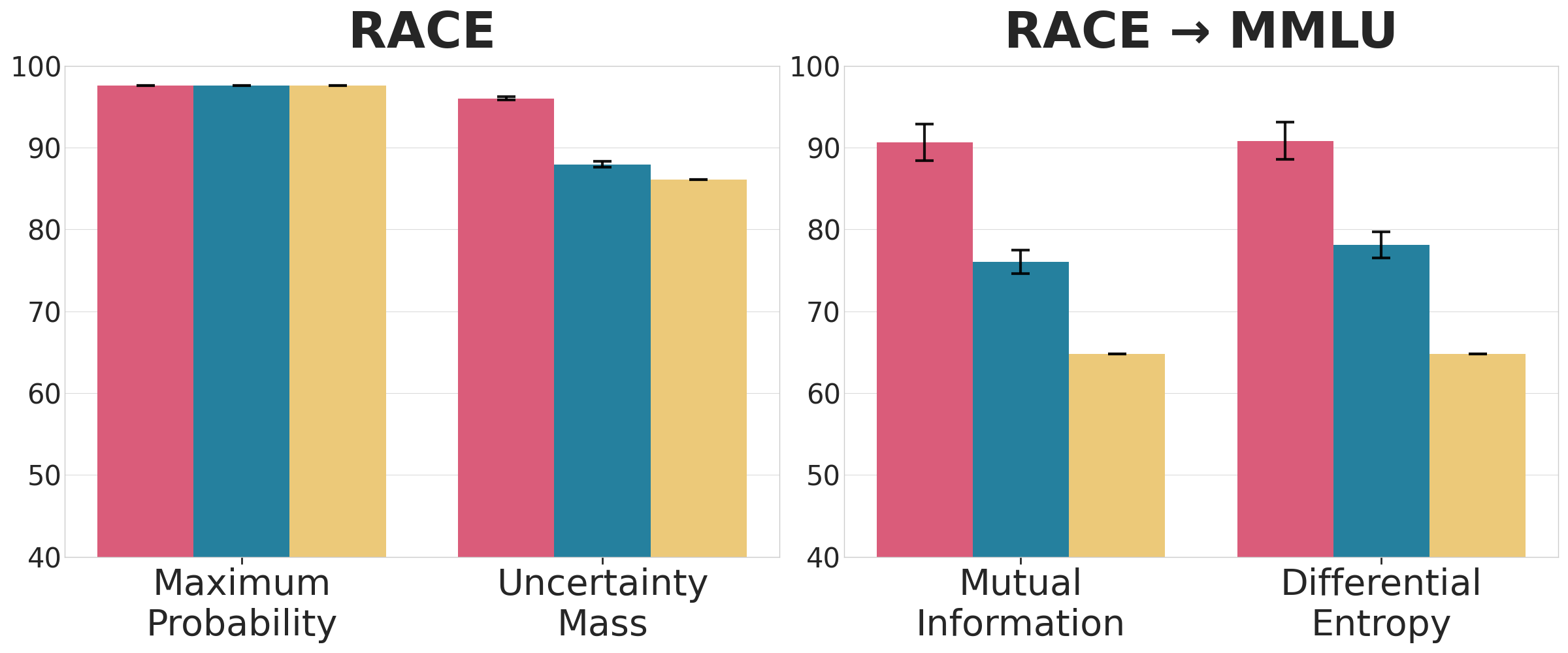}
    \end{minipage}\hfill
    \caption{Comparison of uncertainty estimation performance based on different transformation methods.}
    \label{fig:trans}
\end{figure*}

\begin{figure*}[ht]
    \centering
    \begin{minipage}{0.25\textwidth}
        \centering
        \includegraphics[width=\linewidth]{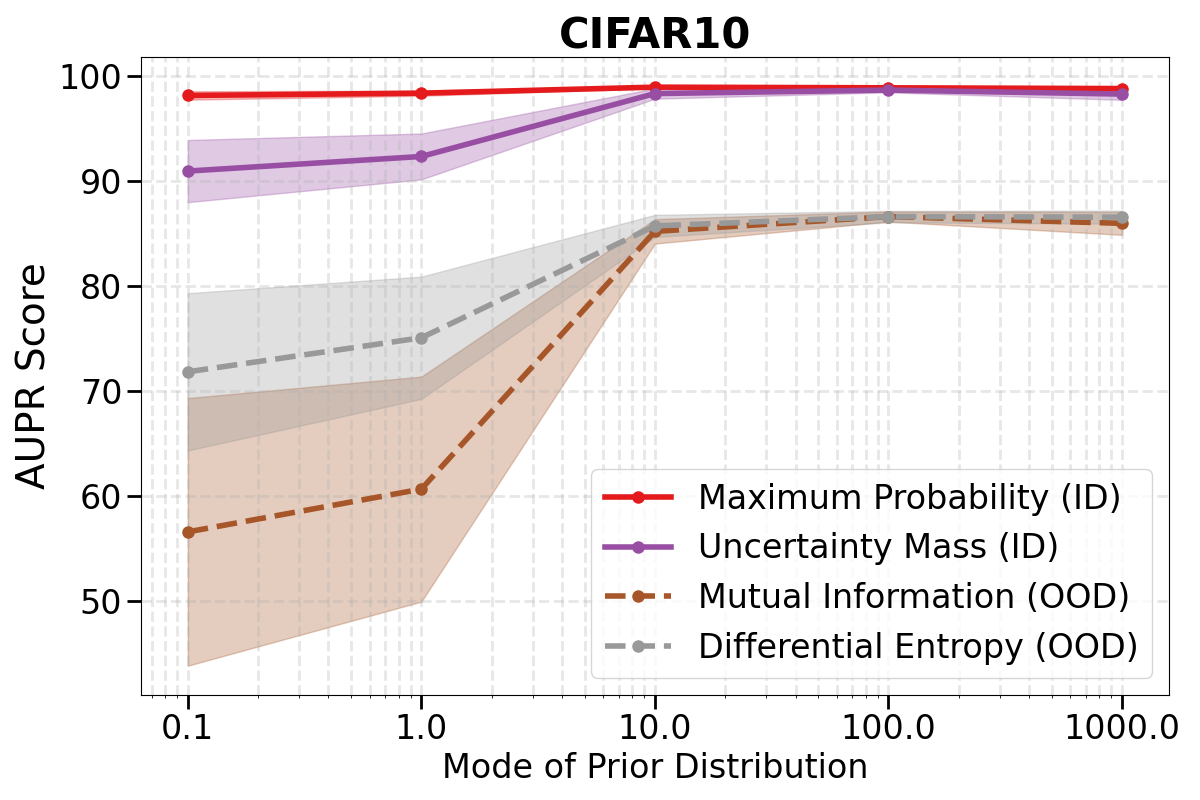}
    \end{minipage}\hfill
    \begin{minipage}{0.25\textwidth}
        \centering
        \includegraphics[width=\linewidth]{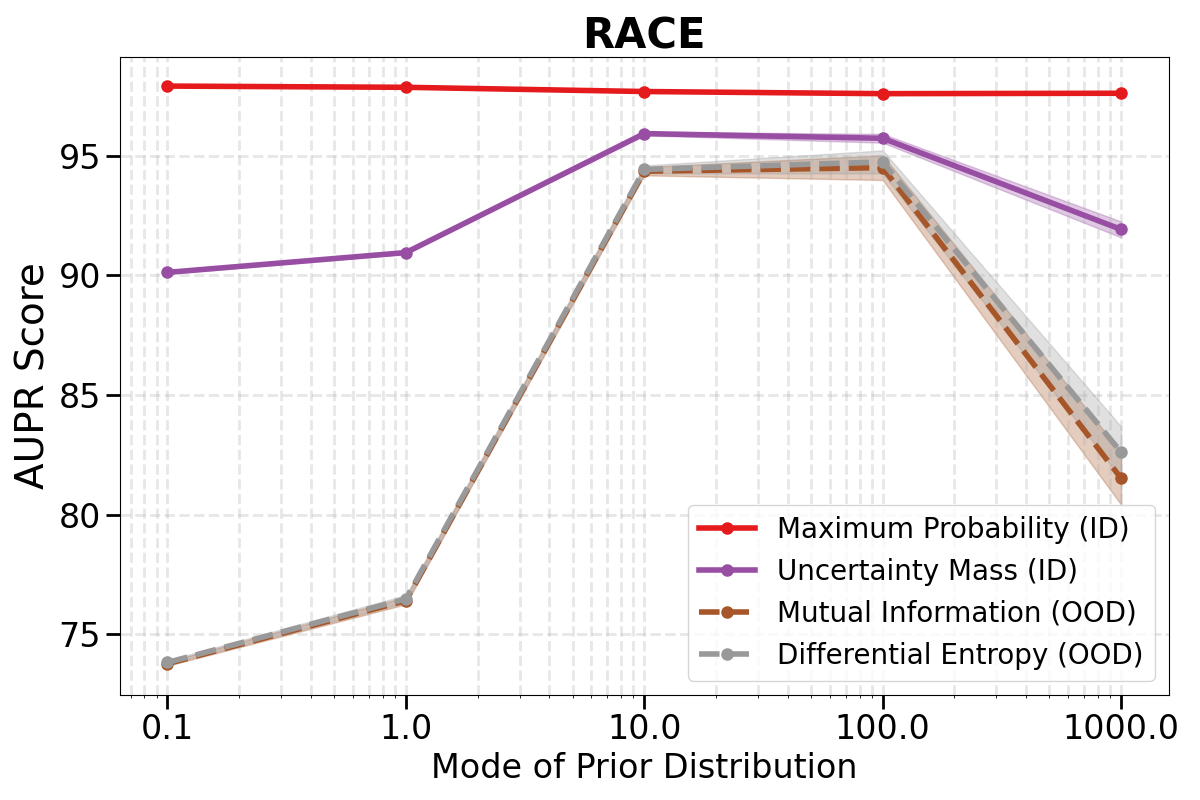}
    \end{minipage}\hfill
    \begin{minipage}{0.25\textwidth}
        \centering
        \includegraphics[width=\linewidth]{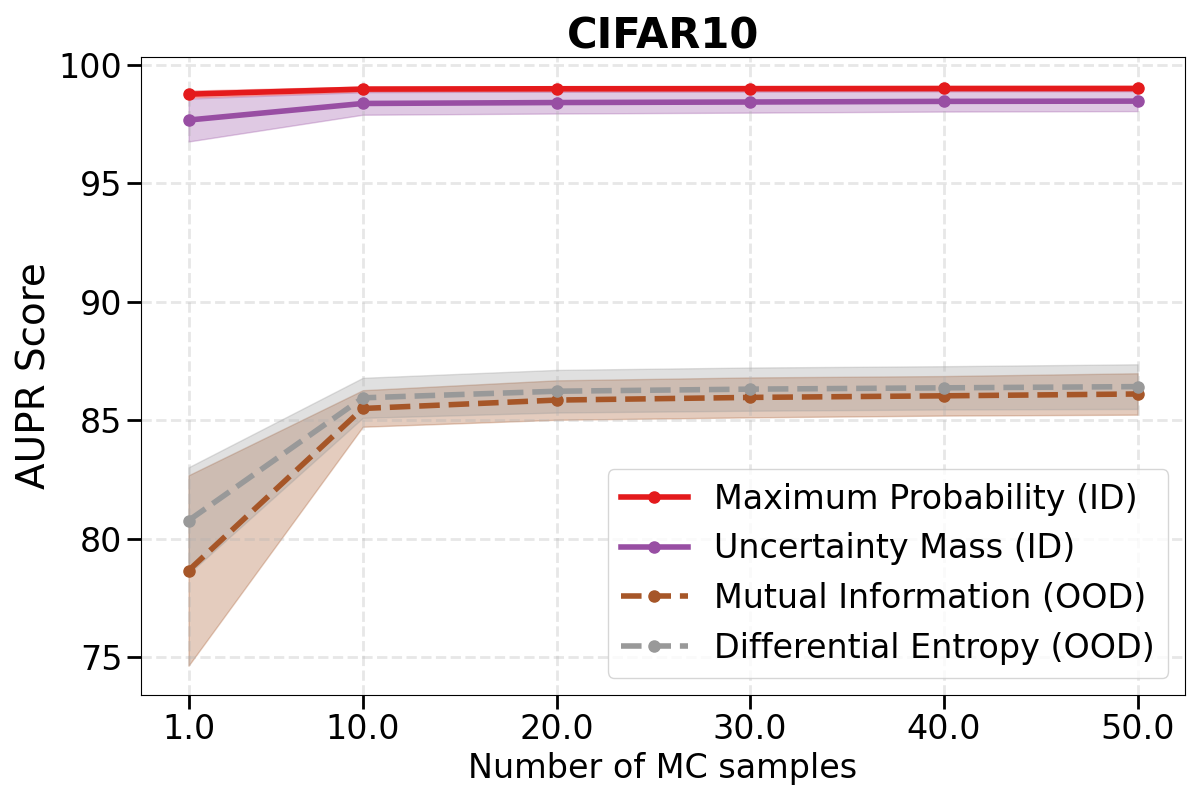}
    \end{minipage}\hfill
    \begin{minipage}{0.25\textwidth}
        \centering
        \includegraphics[width=\linewidth]{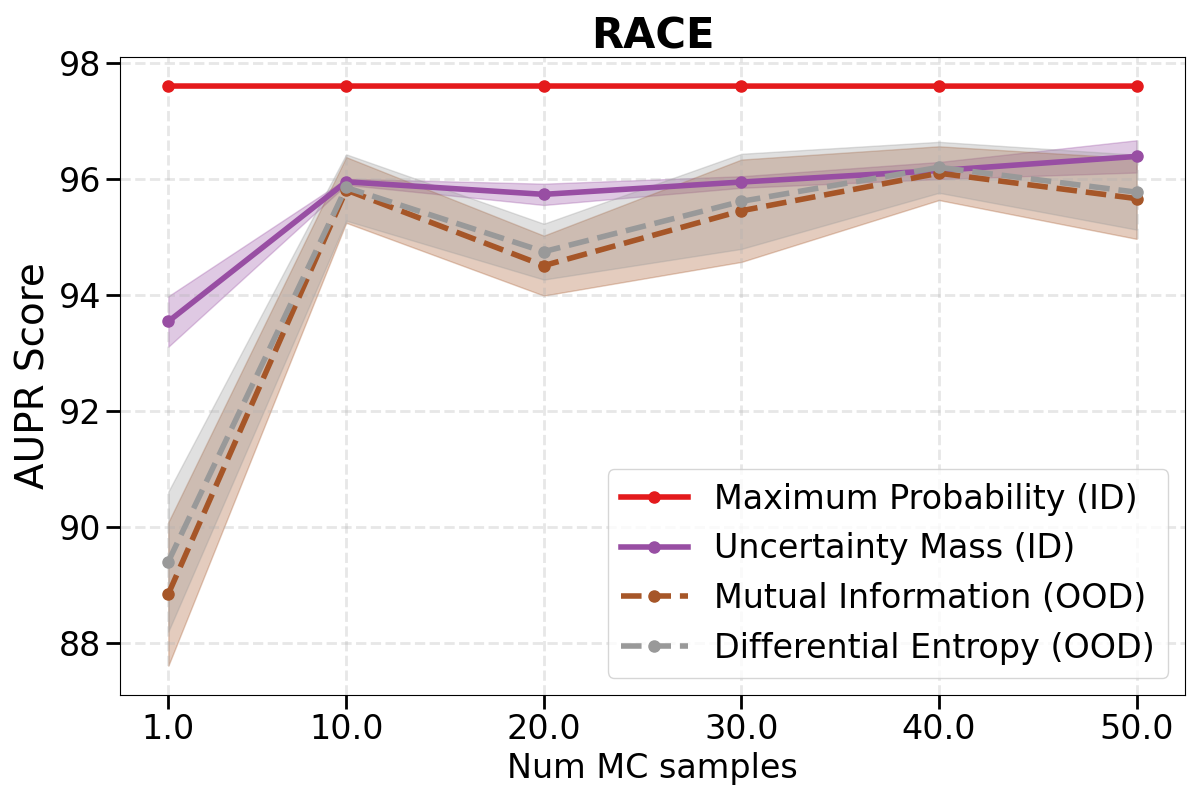}
    \end{minipage}\hfill
    \caption{Ablation studies on each AUPR scores on different parameters of prior distribution (left) and different number of MC samples (right). 95\% confidence interval is shaded.}
    \label{fig:ablation}
\label{fig:mode_auroc}
\end{figure*}


\subsection{Image Classification Results}
Table~\ref{tab:main2} summarizes results on CIFAR-10 and ImageNet.
On CIFAR-10, ETN consistently outperforms all baselines in uncertainty estimation while preserving accuracy.
In contrast, other methods exhibit noticeable accuracy drops, suggesting overfitting to the relatively small dataset used for post-hoc adaptation compared to the pretraining data.

A similar pattern emerges on ImageNet. Most baselines lose pretrained accuracy, with DMM being the most affected. Despite extensive hyperparameter tuning (e.g., batch size, learning rate, and architecture), DMM failed to improve beyond the results reported in Table~\ref{tab:main2}.
Since DMM introduces the largest number of additional trainable parameters among all baselines, these findings indicate that the dataset–model size mismatch is the primary source of performance degradation~\citep{Bahri_2024,hestness2017deeplearningscalingpredictable}.

We also note that the Laplace approximation achieves competitive OOD detection performance in MI, but at the cost of reduced accuracy and near-zero MP for confidence estimation.
Overall, none of the baselines provide competitive performance across both confidence estimation and OOD detection, whereas ETN consistently performs well on both.


\subsection{LLM Results}

Table~\ref{tab:main3} presents results of RACE and OBQA. 
We observe trends consistent with image classification: most uncertainty estimation methods that fine-tune the model directly degrade its original accuracy, whereas ETN preserves accuracy.

Among the baselines, MC-Dropout performs competitively in both confidence estimation and OOD detection.
However, as shown in Figure~\ref{fig:frontier}, it sacrifices a significant inference runtime for uncertainty estimation, limiting its practical applicability.
Among the methods employing EDL-style uncertainty estimation, ETN consistently achieves the best results across nearly all settings, demonstrating that even simple scalar scaling is sufficient to transform categorical predictions into effective evidential distributions.

\subsection{Comparison of Transformation Methods}
\label{sec:compare_trans_method}
We further compare three transformation strategies that preserve accuracy: ETN, static scaling, and AdaTS.  
Results are shown in Figure~\ref{fig:trans}. On ImageNet, static scaling and AdaTS shows slightly higher maximum probability than ETN. 
However, they suffer substantial drops (about 17\% and 10\%, respectively) in Uncertainty Mass, and both methods underperform in OOD detection by a significant margin, with 21.5\% and 25.5\% for mutual information, respectively. 
This indicates that explicitly modeling a distribution over the transformation parameter, and treating the transformed evidential output as a posterior, is more effective than using a single deterministic parameter.
Moreover, the fact that both ETN and AdaTS outperform static scaling shows that sample-dependent transformation is essential for constructing meaningful evidential distributions. 

\subsection{Ablation Study}
\paragraph{Prior Parameters.} We begin by examining how model performance varies with different choices of prior parameters for \(p(A)\). Increasing the mode of \(p(A)\) shifts the distribution toward larger values of \(A\), which in turn enlarges the logit margin—consistent with the behavior characterized in Corollary~\ref{cor:softplus-equal-ce}. Figure~\ref{fig:mode_auroc} presents the resulting AUPR trends on CIFAR-10 and RACE. In both settings, all AUPR metrics consistently improve as the mode increases. This suggests that larger prior modes encourage greater inter-class margins, enabling the transformed Dirichlet parameters to better capture evidential distributions and thereby enhance uncertainty estimation quality.

However, an excessively large mode for RACE leads to degraded performance, likely due to the $(C-1)$ term in Equation~\ref{eq:softplus-edl-lb-equal-ce}, which becomes much larger for LLMs (e.g., $C\sim 10^5$ for Llama-3.1) than for CIFAR-10 classifiers ($C=10$),  thereby relaxing the condition in Corollary~\ref{cor:softplus-equal-ce} for LLMs. 

\paragraph{Number of Monte-Carlo samples.}
We also test how number of Monte-Carlo (MC) samples of $A$ affects ETN's performance. As the number of samples increases from 1 to 10, performance improves, confirming that the variational distribution is not collapsed into a Dirac-delta and that sampling multiple $A$ values provides better estimates.
Beyond 10 samples, performance plateaus, indicating that additional samples offer little benefit and that our method remains computationally efficient.

\section{Conclusion}
In this work, we propose the \textbf{Evidential Transformation Network (ETN)}, an efficient approach for post-hoc uncertainty estimation.
Experiments on both image classification and large language models show that ETN consistently outperforms other post-hoc methods, while preserving accuracy and adding minimal inference latency.
We hope that ETN contributes to bridging the gap between the practicality and trustworthiness of pretrained models.

\section{Acknowledgement}
This work was supported by Institute for Information \& communications Technology Promotion(IITP) grant funded by the Korea government(MSIT). This work was partly supported by ICT Creative Consilience Program through the Institute of Information \& Communications Technology Planning \& Evaluation(IITP) grant funded by the Korea government(MSIT). This work was supported by Institute of Information \& communications Technology Planning \& Evaluation(IITP) under the Leading Generative AI Human Resources Development(IITP-2025-R2408111) grant funded by the Korea government(MSIT).

{
    \small
    \bibliographystyle{ieeenat_fullname}
    \bibliography{main}
}

\clearpage
\setcounter{page}{1}
\maketitlesupplementary

\section{Limitations}
While ETN improves the uncertainty estimation performance of pretrained models without harming accuracy and with only minimal additional computational cost, it also has several limitations.

First, the benefits of ETN are largely empirical rather than theoretical. Recent works have raised concerns about EDL from a theoretical standpoint, arguing that current training procedures do not guarantee a faithful modeling of epistemic uncertainty~\cite{shen2024uncertainty,bengs2023secondorderscoringrulesepistemic,juergens2024epistemic}. Our observation that simple scalar scaling is often sufficient to make logits suitable as Dirichlet parameters may reflect inherent limitations in existing EDL training formulations. However, we do not think that this empirical success should be underestimated, as robust and consistent improvements across diverse datasets and architectures are precisely what is required for practical deployment of uncertainty-aware pretrained models, even in the absence of a complete theoretical account.

Second, our method requires access to the logits and the last hidden representation of the pretrained model, which may not be available when using closed-source models exposed only through an API (e.g., recent GPT models). Nevertheless, since ETN depends solely on these two quantities—unlike many uncertainty estimation baselines that require access to the full model architecture or gradients—it remains relatively compatible with \emph{gray-box models}~\cite{barshalom2025tokenprobabilitieslearnablefast,hiranandani2025logitsneedadaptclosed}.

\section{Proofs and Derivations}
In this section, we analyze the behavior of logits produced by models trained with cross-entropy and EDL losses. We first define the softmax per-sample $(x,y)$ cross-entropy loss 
as:
\[
\mathcal{L}_{\text{CE}}(\bm z,y) \;=\; -\log \frac{e^{z_y}}{\sum_{j=1}^C e^{z_j}}
= \log\!\Big(1+\sum_{j\ne y} e^{\,z_j-z_y}\Big) 
\]

Then we define the inter-class margin of an sample as:
\[
\gamma(\bm z,y)=z_y-\max_{j\ne y}z_j
\]


Given these definitions, we now present two lemmas characterizing the relationship between cross-entropy and EDL models.

\begin{lemma}[Zero loss implies infinite margin]
\label{lem:margin}
Cross-entropy loss becomes zero if and only if the margin between the logits of the label and other logits become infinite. i,e:

\[
\mathcal{L}_{\text{CE}}(\bm z,y)\to 0 
\quad \Longleftrightarrow \quad
\gamma(\bm z, y) \to \infty .
\]
\end{lemma}

\begin{proof} Suppose $\mathcal{L}_{\text{CE}}(\bm z,y)\le \varepsilon$. Then the softmax probability of the correct class satisfies
\[
\frac{e^{z_y}}{\sum_j e^{z_j}} \;\ge\; e^{-\varepsilon}.
\]
Rearranging gives
\[
\sum_{j\ne y} e^{z_j} \;\le\; e^{z_y}(e^\varepsilon-1).
\]
Hence, for each $j\ne y$,
\[
z_y - z_j \;\ge\; -\log(e^\varepsilon-1).
\]
Since $-\log(e^\varepsilon-1)\to\infty$ as $\varepsilon\rightarrow 0$, the margin diverges.

\noindent Conversely, if $\gamma(\bm z, y)\to\infty$, then $z_j-z_y\to -\infty$ for each $j\ne y$, so $e^{z_j-z_y}\to 0$. Therefore

\[
\mathcal{L}_{\text{CE}}(\bm z,y) = \log\!\Big(1+\sum_{j\ne y} e^{z_j-z_y}\Big)\to 0
\]
\end{proof}

\begin{lemma}[Margin of EDL models]
\label{thm:equal-ce-loss-bounds}
For a sample $(x,y)$, assume there exists $\eta$ with $0\le \eta<\nu-b_y$ such that
\[
\alpha_y \ \ge\ \nu-\eta,
\qquad
\alpha_j \ \le\ b_j+\eta\quad \forall j\neq y.
\]
Then the inter-class margin of an sample $(x,y)$ of EDL models is defined by:
\begin{equation}
\label{eq:edl-LB-equal-ce}
\gamma_{\mathrm{EDL}}(\bm z,y)=
f^{-1}(\nu-b_y-\eta)\;-\;f^{-1}(\eta)
\end{equation}
\end{lemma}

\begin{proof}
From $\alpha_y\ge \nu-\eta$ we get

\begin{equation}
\begin{aligned}
f(z_y) &= \alpha_y - b_y \ \ge\ \nu - b_y - \eta \\
       &\Longrightarrow 
    z_y \ge f^{-1}(\nu - b_y - \eta).\nonumber
\end{aligned}
\end{equation}
For any $j\ne y$, the assumption $\alpha_j\le b_j+\eta$ gives
\[
f(z_j)=\alpha_j-b_j\ \le\ \eta
\quad\Longrightarrow\quad
z_j\ \le\ f^{-1}(\eta).
\]
Taking the maximum over $j\ne y$ yields
$\max_{j\ne y} z_j \le f^{-1}(\eta)$, hence
\[
\gamma_{\mathrm{EDL}}(\bm z,y)
= z_y-\max_{j\ne y} z_j
\ \ge\ f^{-1}(\nu-b_y-\eta)-f^{-1}(\eta),
\]
which is Equation~\ref{eq:edl-LB-equal-ce}.
\end{proof}

\subsection{Proof to Proposition 1} By Lemma~\ref{lem:margin}, zero CE loss is achieved by sending the margins $\gamma(\bm z, y)\to\infty$, which can be done by either pushing the correct logit up or the incorrect logits down. Given this, we provide two explicit cases of logits that both show vanishing cross-entropy loss but lead to bounded and diverging values of $\alpha_0$, respectively.
\paragraph{Bounded $\bm{\tilde{\alpha}_0}$.}
Set $\tilde z_y=0$ and $\tilde z_{j\ne y}=-t$.
Then $\mathcal{L}_{\rm CE}(\tilde{\bm z},y)=\log(1+(C-1)e^{-t})\to0$.
Therefore,
\begin{equation}
\begin{aligned}
\tilde{\alpha}_0
&= \big(f(0)+b\big)+\sum_{j \ne y}\big(f(-t)+b\big) \\
&\to \; \big(f(0)+b\big) + (C-1)b \;<\; \infty,
\end{aligned}
\end{equation}
as $t\to\infty$ since $f(-t)\to0$.

\paragraph{Diverging $\bm{\hat{\alpha}_0}$.}
Set $\hat z_y=t$ and $\hat z_{j\ne y}=0$.
Then $\mathcal{L}_{\rm CE}(\hat{\bm z},y)=\log(1+(C-1)e^{-t})\to0$, and
\begin{equation}
\alpha_0(\hat{\bm z})
= \big(f(t)+b\big) + \sum_{j\ne y}\big(f(0)+b\big)
\;\xrightarrow\; \infty,
\end{equation}
as $t\to\infty$ since $f(t)\to\infty$.

\subsection{Proof of Theorem 1}

For a sample $(x, y)$, assume $L := \mathcal{L}_{\mathrm{CE}}(\bm z, y) = \mathcal{L}_{\mathrm{EDL}}(\bm z, y)$. Since $z_j - z_y \le -\gamma_{\mathrm{CE}}(\bm z, y)$ for all $j \ne y$, we obtain an upper bound on the CE loss:
\begin{equation}
\label{eq:ce_ub_revised}
\begin{aligned}
L
&= \log\!\left(1 + \sum_{j \ne y} e^{z_j - z_y}\right) \\
&\le \log\!\left(1 + (C - 1)\, e^{-\gamma_{\mathrm{CE}}(\bm z, y)}\right).
\end{aligned}
\end{equation}

\noindent Let the lower bound on the EDL margin be
\[
\gamma_{\mathrm{LB}}
:= f^{-1}(\nu - b_y - \eta)\;-\; f^{-1}(\eta).
\]
\noindent Assume further that $L$ satisfies
\begin{equation}
\label{eq:lower_bound_revised}
L \ge \log\!\left(1 + (C - 1) e^{-\gamma_{\mathrm{LB}}}\right).
\end{equation}

\noindent Combining Equation~\ref{eq:ce_ub_revised} and Equation~\ref{eq:lower_bound_revised}, we obtain
\[
\log\!\left(1 + (C - 1) e^{-\gamma_{\mathrm{CE}}(\bm z, y)}\right)
\;\ge\;
\log\!\left(1 + (C - 1) e^{-\gamma_{\mathrm{LB}}}\right).
\]

\noindent Since the logarithm is monotone increasing and $(C - 1) > 0$, it follows that
\[
\gamma_{\mathrm{CE}}(\bm z, y)
\;\le\;
\gamma_{\mathrm{LB}},
\]
and therefore,
\begin{equation}
\label{eq:margin_comparison_revised}
\gamma_{\mathrm{EDL}}(\bm z, y)
\;\ge\;
\gamma_{\mathrm{CE}}(\bm z, y).
\end{equation}

\noindent Define event $A$ as the event that Equation~\ref{eq:lower_bound_revised} holds, and event $B$ as the event that Equation~\ref{eq:margin_comparison_revised} holds.
From the derivation above, we have $A \subseteq B$, which implies
$P(A) \le P(B)$~\citep{evans2004probability}.
Thus,
\begin{equation}
\label{eq:margin_bound}
\begin{aligned}
&P\!\left(\gamma_{\mathrm{EDL}}(\bm z, y) \ge \gamma_{\mathrm{CE}}(\bm z, y)\right) \\
&\quad\ge\;
P\!\Big(
L \ge
\log\!\Big(
1 + \tfrac{C - 1}{e^{f^{-1}(\nu - b_y - \eta) - f^{-1}(\eta)}}
\Big)
\Big).
\end{aligned}
\end{equation}

\subsection{Proof to Corollary 1}
With $f$ as \emph{softplus}, $f^{-1}(x)=\log(e^{x}-1)$. Plugging into Equation~\ref{eq:margin_bound}, we get:
\begin{equation}
\label{eq:softplus-edl-lb-equal-ce}
\begin{aligned}
&P\big(\gamma_{\mathrm{EDL}}(\bm z,y)\ \ge\ \gamma_{\mathrm{CE}}(\bm z,y)\big) \\
&\quad \ge\
P\!\Big(
L
\ge 
\log\big(1 + (C - 1)\tfrac{e^{\eta} - 1}{e^{\nu - b_y - \eta} - 1}\big)
\Big).\nonumber
\end{aligned}
\end{equation}

\begin{algorithm}[t]
\caption{Training and Inference of Evidential Transformation Network}
\label{alg:train_and_infer}
\begin{footnotesize}
\begin{algorithmic}[1]
\Require Dataset $\mathcal{D}=\{(x_i,y_i)\}_{i=1}^{N}$, pretrained model $\theta=h\circ\phi$, number of MC samples $M$, monotonically increasing function $f$, 
\Parameters Evidential Transformation Network $\theta_{\mathrm{ETN}}$, prior belief term $\bm b$
  \For{$(x, y)\in \mathcal{D}$}  \Comment{Loop over dataset}
    \State $\theta_A \gets \theta_{\mathrm{ETN}}\!\big(\phi(x)\big)$ \Comment{{\scriptsize Compute parameters for variational distribution}}
    \State $\bm z \gets \theta(x)$ \Comment{Compute logits for sample $x$}
    \State $\mathcal{P} \gets \emptyset$
    \For{$m \gets 1$ \textbf{to} $M$}
      \State $A^{(m)} \sim 
      Dist(\theta_A)$
      \Comment{{\scriptsize Sample from variational distribution}}
      \State $\bm\alpha' \gets f\!\big(A^{(m)}\bm z\big) + \bm b$
      \State $p'=\mathbb{E}_{\bm{\pi}\sim \mathrm{Dir}(\bm{\alpha'})}\!\left[p(y\mid \bm{\pi})\right]$
      \State $\mathcal{P} \gets \mathcal{P} \cup \{p'\}$
    \EndFor
    \State $\bar p' \gets \frac{1}{M}\sum_{p' \in \mathcal{P}} p'$
    \If{\textsc{Training}}
      \State \textbf{Backprop} through $\mathcal{L}_{\mathrm{ETN}}\!\big(\theta_{\mathrm{ETN}}\big)$
    \Else
      \State \textbf{return} $\displaystyle  \arg\max\limits_{y}\; \bar p'$
    \EndIf
  \EndFor
\end{algorithmic}
\end{footnotesize}
\end{algorithm}

\section{Modeling Transformation Parameterizations} 
\label{sec:dimension}
In this section, we describe how the transformation parameter $A$ is modeled when defined as a scalar, vector, or matrix. Specifically, we explain 
(1) how the variational distribution over $A$ is constructed, and  
(2) how the prior term $\bm b$ is handled.  
For clarity, we denote the scalar case by $a$, the vector case by $\bm a$, and the matrix case by $\bm A$.

\paragraph{Scalar (\(a\in\mathbb{R}_+\)).}
We constrain \(a>0\) and model it with a Gamma distribution:
\[
a \sim \mathrm{Gamma}(\alpha^{\mathrm G},\,\beta^{\mathrm G}),
\]
where the shape $\alpha^{\mathrm G}$ and rate $\beta^{\mathrm G}$ are predicted by ETN.
To strictly preserve accuracy, we set all elements of $\bm b$ to be identical, i.e.,
\[
b_1 = b_2 = \dots = b_C.
\]
\paragraph{Vector (\(\bm a \in \mathbb{R}^C_+\)).} 
We model \(\bm a\) as a product of Gamma distributions, one per class:
\[
\bm a \sim \prod_{i=1}^{C}{\mathrm{Gamma}(\alpha_i^{\mathrm G},\beta^{\mathrm G}_i)}
\]
ETN predicts the shape \(\bm \alpha^{\mathrm G}=(\alpha^{\mathrm G}_1,\dots,\alpha^{\mathrm G}_C)^\top\) and rate \(\bm \beta^{\mathrm G} = (\beta^{\mathrm G}_1,\dots,\beta^{\mathrm G}_C)^\top\).
For $\bm b$, we treat each element independently and train them separately.

\paragraph{Matrix (\(\bm A\in\mathbb{R}^{C\times C}\)).} Matrix transformation are a natural choice since they directly operate in Dirichlet space~\citep{kull2019beyond}. Although the Wishart distribution\cite{wishart1928generalised} would be a natural distribution for positive-definite matrices, in practice we found its parameterization too restrictive and its reparamterization unstable during training. Instead, we model the flattened matrix as a Gaussian with Kronecker-factored covariance:

\[
\mathrm{vec}(\bm A)\;\sim\;\mathcal{N}(\boldsymbol\mu,\;\bm\Sigma), 
\]
where $\bm\Sigma=\bm B\otimes \bm D,\quad$ with $\bm B=\bm L_B\bm L_B^\top$ and $\bm D=\bm L_D\bm L_D^\top$. ETN predicts \(\boldsymbol\mu,\bm L_B,\bm L_D\). 
To encourage monotonic behavior, we apply a \emph{softplus} to the diagonal elements of sampled $\bm A$, keeping off-diagonal terms unconstrained.
The prior $p(\bm A)$ is set as a Gaussian with mode and variance matching the scalar and vector Gamma priors. Additionally, we adopt the ODIR (Off-Diagonal and Intercept Regularization) loss\cite{kull2019beyond} on $\bm \mu$ for stable optimization.
As in the vector case, all elements of $\bm b$ are treated independently and trained separately.

\section{Experimental Setting}

\subsection{Training Details}
The hyperparameters used for training ETN are summarized in Table~\ref{tab:hyper_train}. For LLM experiments, we employ cosine learning-rate scheduling with warm-up steps. All experiments are performed using three different random seeds, and we report the mean along with 95\% confidence intervals.

For post-hoc uncertainty estimation baselines, we select the checkpoint that achieves the highest accuracy on the adaptation dataset. In contrast, for ETN with scalar scaling, we select the checkpoint with the lowest loss on the adaptation dataset.

All training and inference are performed using eight NVIDIA A6000 GPUs.

\subsection{Architecture of Evidential Transformation Network} \label{appen:etn_arch}
The network is composed of independent modules, each predicting a parameter of the variational distribution. (e.g., for a scalar-prediction case, the network contains two modules to predict two parameters, $\alpha^{\mathrm G}$ and $\beta^{\mathrm G}$, respectively.). In the image classification case, each module is implemented as a 2-layer MLP with hidden dimension 256. For LLMs, each module is implemented as a 3-layer MLP with hidden dimension 512.

Moreover, the training and inference procedures of ETN are outlined in Algorithm~\ref{alg:train_and_infer}.

\subsection{Datasets}
\paragraph{CIFAR-10.}
Since CIFAR-10 does not include an official validation split, we use 5\% of the original training set for post-hoc uncertainty adaptation and the remaining 95\% for pretraining the VGG16 model. Evaluation is conducted on the CIFAR-10 test set, as well as the SVHN and CIFAR-100 test sets for OOD assessment.

\paragraph{ImageNet.}
Following \citet{minderer2021revisiting}, we use 20\% of the \texttt{ILSVRC\_2012} validation set for post-hoc adaptation and the remaining 80\% for evaluation. For ImageNet-A, ImageNet-S, and ImageNet-R, we use all available samples from each subset.

\paragraph{RACE.}
To ensure that RACE serves as in-distribution data, we train LLMs on the official training set using cross-entropy loss. The validation set is used to adapt all post-hoc uncertainty estimation methods, and the test set is used exclusively for evaluation.

\paragraph{OBQA.}
Similar to RACE, we treat OBQA as in-distribution by training LLMs on the official training set with cross-entropy loss. We use the validation set for post-hoc adaptation and the test set for evaluation.

\paragraph{MMLU.}
We use three domains from MMLU, adopting the same subsets as in \citet{yangbayesian}. The selected domains and their corresponding subsets are listed in Table~\ref{table:mmlu_domains}.

\renewcommand{\arraystretch}{1.0}

\begin{table}[t]
\centering
\scriptsize
\resizebox{1.0\linewidth}{!}{
\begin{tabular}{lcccc}
\toprule
\textbf{Setting} & \textbf{VGG16} & \textbf{ResNet50} & \textbf{Llama-3.1-8B} & \textbf{Gemma-2-9B} \\
\midrule
\multicolumn{5}{l}{\textbf{\textit{Pretrain}}} \\[-1pt]
\quad Batch size & 1024 & -- & 4 & 4 \\
\quad Learning rate & \(2.5\times10^{-4}\) & -- & \(2.5\times10^{-4}\) & \(2.5\times10^{-4}\) \\
\quad Epochs & 200 & -- & 3 & 3 \\
\midrule
\multicolumn{5}{l}{\textbf{\textit{Uncertainty adaptation}}} \\[-1pt]
\quad Batch size & 1024 & 64 & 8 & 8 \\
\quad Learning rate & \(1\times10^{-3}\) & \(1\times10^{-3}\) & \(1\times10^{-3}\) & \(1\times10^{-3}\) \\
\quad Epochs & 50 & 50 & 5 & 5 \\
\midrule
\multicolumn{5}{l}{\textbf{\textit{ETN}}} \\[-1pt]
\quad Prior mode & 10 & 5 & 100 & 100 \\
\quad Prior variance & 5 & 5 & 5 & 5 \\
\quad MC samples & 20 & 20 & 20 & 20 \\
\quad $\lambda$ & 1 & \(1\times10^{-3}\) & 1 & 1 \\
\quad $\nu$ & \(1\times10^{4}\) & \(1\times10^{4}\) & \(1\times10^{4}\) & \(1\times10^{4}\) \\
\bottomrule
\end{tabular}}
\caption{Training and hyperparameter settings for each model.}
\label{tab:hyper_train}
\end{table}

\subsection{Models}
\paragraph{VGG16.}
We adopt the VGG16 architecture~\citep{simonyan2015very}, which is composed of 16 convolutional layers followed by 3 fully connected layers. Batch normalization is applied to all convolutional layers. All parameters are updated during the pretraining stage, and for baselines that rely on training the original model ($\text{MAP}_{\text{CE}}$, $\text{MAP}_{\text{EDL}}$ and IB-EDL), all parameters are likewise fully fine-tuned.

\paragraph{ResNet50.}
We use the ResNet50 architecture~\citep{he2015deepresiduallearningimage}, a 50-layer convolutional network organized into five \emph{stages}, each containing multiple residual blocks operating at a fixed spatial resolution and channel width. For baselines that require training the original model, all parameters are fully fine-tuned.

\paragraph{Llama-3.1-8B} For baselines that require tuning the original pretrained model, we applied LoRA to all attention layers with a rank of 8 and lora alpha value to 16, and trained only the LoRA layers, following the setting in \citet{yangbayesian,li2025calibrating}. 

\paragraph{Gemma-2-9B} We use the identical setting as Llama-3.1-8B.

\begin{table}[t]
\centering
\begin{tabular}{l}
\toprule
\textbf{Domain and Subsets of MMLU} \\
\midrule
\textbf{Computer Science:} \\
\quad \texttt{college\_computer\_science}\\ 
\quad \texttt{computer\_security}\\ 
\quad \texttt{high\_school\_computer\_science}\\
\quad \texttt{machine\_learning} \\[2pt]
\textbf{Engineering:} \\
\quad \texttt{electrical\_engineering} \\[2pt]
\textbf{Math:} \\
\quad \texttt{college\_mathematics} \\
\quad \texttt{high\_school\_mathematics} \\
\quad \texttt{abstract\_algebra} \\
\bottomrule
\end{tabular}
\caption{MMLU domains and their corresponding subsets.}
\label{table:mmlu_domains}
\end{table}
\begin{table*}[h!]
    \centering
    \label{tab:placeholder}
    \resizebox{1.0\linewidth}{!}{
    \begin{tabular}{lccccc}
        \toprule
        Method & CIFAR10 $\to$ CIFAR10-OOD & ImageNet $\to$ ImageNet-OOD & OBQA $\to$ MMLU & RACE $\to$ MMLU \\
        \midrule
        MD   & \result{45.43}{1.2} / \result{56.69}{1.06} & \result{\textbf{87.54}}{0.16} / \result{77.45}{0.24}  & \result{70.5}{0.04} / \result{54.42}{0.05} & \result{87.28}{0.01} / \result{54.44}{0.02} \\
        ODIN & \result{\textbf{86.41}}{0.99} / \result{\textbf{87.29}}{0.82} & \result{79.77}{0.00} / \result{72.11}{0.00}   & \result{61.35}{0.00} / \result{50.11}{0.00} & \result{81.22}{0.00} / \result{49.65}{0.00} \\
        \rowcolor{gray!15}
        ETN & \result{85.93}{0.92}/\result{86.5}{0.97} & \result{84.78}{0.36} / \result{\textbf{79.86}}{0.49} & \result{\textbf{91.7}}{0.00} / \result{\textbf{83.39}}{0.01} & \result{\textbf{96.80}}{0.39} / \result{\textbf{87.57}}{0.01} \\ 
        \bottomrule
    \end{tabular}}
    \caption{Comparison of ETN to OOD-detection methods. We showcase both AUPR and AUROC scores, respectively. For ETN, we outline the scores based on maximum probability.}
    \label{tab:ood_baseline}
\end{table*}

\begin{table*}[t]
\centering
\resizebox{\linewidth}{!}{%
\begin{tabular}{ccccccccccccc}
\toprule
\multirow{2}{*}{Method} &
\multicolumn{2}{c}{CIFAR-10} &
{$\to$ CIFAR-OOD} &
\multicolumn{2}{c}{ImageNet} &
{$\to$ ImageNet-OOD} &
\multicolumn{2}{c}{RACE} &
{$\to$ MMLU} &
\multicolumn{2}{c}{OBQA} &
{$\to$ MMLU} \\
\cmidrule(lr){2-3}\cmidrule(lr){4-4}\cmidrule(lr){5-6}\cmidrule(lr){7-7}
\cmidrule(lr){8-9}\cmidrule(lr){10-10}\cmidrule(lr){11-12}\cmidrule(lr){13-13}
&
ACC & UE & UE &
ACC & UE & UE &
ACC & UE & UE &
ACC & UE & UE \\
\midrule
DUQ
& \result{42.25}{9.4} & \result{42.05}{9.3} & \result{54.86}{8.0}
& \result{0.09}{0.0} & \result{0.18}{0.0} & \result{69.77}{1.9}
& \result{21.52}{0.0} & \result{22.54}{0.2} & \result{74.31}{11.7}
& \result{27.60}{0.0} & \result{26.74}{1.1} & \result{50.85}{3.0} \\
SNGP
& \result{83.62}{1.4} & \result{90.75}{1.3} & \result{57.72}{4.7}
& \result{12.83}{1.0} & \result{12.83}{0.9} & \result{65.73}{0.8}
& \result{45.73}{19.5} & \result{49.97}{22.9} & \result{95.11}{1.1}
& \result{39.53}{17.8} & \result{43.74}{20.8} & \result{\textbf{94.74}}{4.3} \\
\rowcolor{gray!15}
ETN
& \result{\textbf{90.70}}{0.0} & \result{\textbf{98.99}}{0.1} & \result{\textbf{85.93}}{0.9}
& \result{\textbf{79.61}}{0.0} & \result{\textbf{88.04}}{0.1} & \result{\textbf{79.86}}{0.5}
& \result{\textbf{89.69}}{0.0} & \result{\textbf{97.60}}{0.0} & \result{\textbf{96.80}}{0.0}
& \result{\textbf{88.80}}{0.0} & \result{\textbf{97.15}}{0.0} & \result{91.70}{0.0} \\
\bottomrule
\end{tabular}%
}
\caption{Comparison of ETN with deterministic uncertainty estimation methods in terms of accuracy (ACC) and uncertainty estimation (UE), where UE is measured by AUPR. For ETN, we report UE based on maximum probability.}
\label{tab:deterministic}
\end{table*}

\subsection{Baselines}

\paragraph{Deep Ensemble (DeepEns).} We use an ensemble of three models in all settings. Each model is trained on the same dataset with a different random data order.

\paragraph{MC-Dropout (MCD).} We set the number of forward passes to 20 for all settings. For image classification setting, we use the dropout layer in the pretrained model with a dropout rate of 0.2, while we use the LoRA dropout layer for LLM with a dropout rate of 0.1. 

\paragraph{Laplace Approximation (LA).}\label{LA} We utilize the \texttt{laplace} 
library proposed in Laplace-redux~\citep{daxberger2021laplace}, which provides a integrated tools for bayesian adaptation of neural networks. As for CIFAR-10, We opted for the best setting proposed in the work, which applies laplace approximation on the last layer of the network with Kronecker-factored Generalized Gauss-Newton (GGN) matrix to the Hessian in a post-hoc manner. For ImageNet setting, we construct GGN matrix with diagonal matrix due to constrained resources. To compute distributional uncertainty, we use Monte Carlo sampling for predictive approximation, with the number of MC samples set to 20.

\paragraph{Laplace LoRA (LL).} We build GGN matrix only on all LoRA layers through Kronecker factorization, following \citet{yangbayesian}.

\paragraph{Dirichlet Meta-Model (DMM).}
For VGG16, we follow the implementation of \citet{shen2022posthocuncertaintylearningusing}.  
For ResNet50, DMM takes the final hidden states from each stage as input, with each module consisting of a max-pooling layer followed by two fully connected layers.  
For LLMs, DMM receives hidden states from all transformer layers, and each module is composed of three fully connected layers and a max-pooling layer.

\paragraph{$\text{MAP}_\text{EDL}$.}
We train the model using the reverse KL formulation of $\mathcal{L}_\text{EDL}$, as reverse KL is known to provide more stable optimization than forward KL, primarily due to its mode-seeking behavior~\cite{NEURIPS2019_7dd2ae7d,shen2024uncertainty}.

\paragraph{IB-EDL.}
We follow the original implementation from~\citet{li2025calibrating} for LLM experiments.
For image classification, we modify the architecture by doubling the dimension of the final layer to model both the mean and variance for each class.

\paragraph{Static scaling.}
We adopt the static scaling approach inspired by \citet{guo2017calibration,niculescu2005predicting,platt1999probabilistic} for all experimental settings, and train the additional parameters using the reverse KL formulation of $\mathcal{L}_{\text{EDL}}$.

\paragraph{AdaTS.}
We use the original implementation from \citet{joy2023sample} for all experimental settings, and train the additional parameters using the reverse KL formulation of $\mathcal{L}_{\text{EDL}}$.

\section{Additional Experiments}
\subsection{OOD-Detection Baselines}
In this section, we compare ETN against ODIN \citep{liang2018enhancing} and the Mahalanobis distance method (MD) \citep{lee2018simple}. 
Although neither ODIN nor MD are strictly uncertainty estimation methods, we include them as they both work in post-hoc manner, and there exists close relationship between uncertainty estimation and OOD detection \citep{gawlikowski2023survey}. 
We report both AUPR and Area Under the Receiver Operating Characteristic Curve (AUROC) metrics, and for ETN we showcase scores based on maximum probability. The results are summarized in Table~\ref{tab:ood_baseline}.

On CIFAR-10 and ImageNet, ODIN and MD achieve higher AUPR scores than ETN, respectively. 
However, on OBQA and RACE, ETN outperforms both baselines across AUPR and AUROC. 
It is also worth noting that MD requires learning class-conditional feature distributions, which becomes resource-intensive as the number of classes grows, while ODIN is highly sensitive to its hyperparameters. 
By contrast, ETN avoids these limitations by operating directly in logit space, providing a lightweight and broadly applicable alternative.

\subsection{Deterministic Deep Neural Network Baselines}
In this section, we compare ETN against deterministic deep neural network baselines that estimate uncertainty using a single model and a single forward pass. Specifically, we use DUQ~\citep{van2020uncertainty} and SNGP~\citep{liu2020simple} as baselines. The comparison results are summarized in Table~\ref{tab:deterministic}.

Across all image classification and QA settings except OBQA $\to$ MMLU, ETN consistently outperforms these baselines in uncertainty estimation without sacrificing accuracy. One possible reason is that DUQ requires a separate learnable weight matrix for each class, while SNGP requires learning a class-wise covariance structure for the Gaussian process. Such additional parameters for post-hoc adaptation can lead to overfitting when only a limited adaptation dataset is available, as also observed for other baselines such as Laplace Approximation and Dirichlet Meta Model.

\begin{figure*}[ht]
    \centering
    \begin{minipage}{0.5\textwidth}
        \centering
        \includegraphics[width=\linewidth]{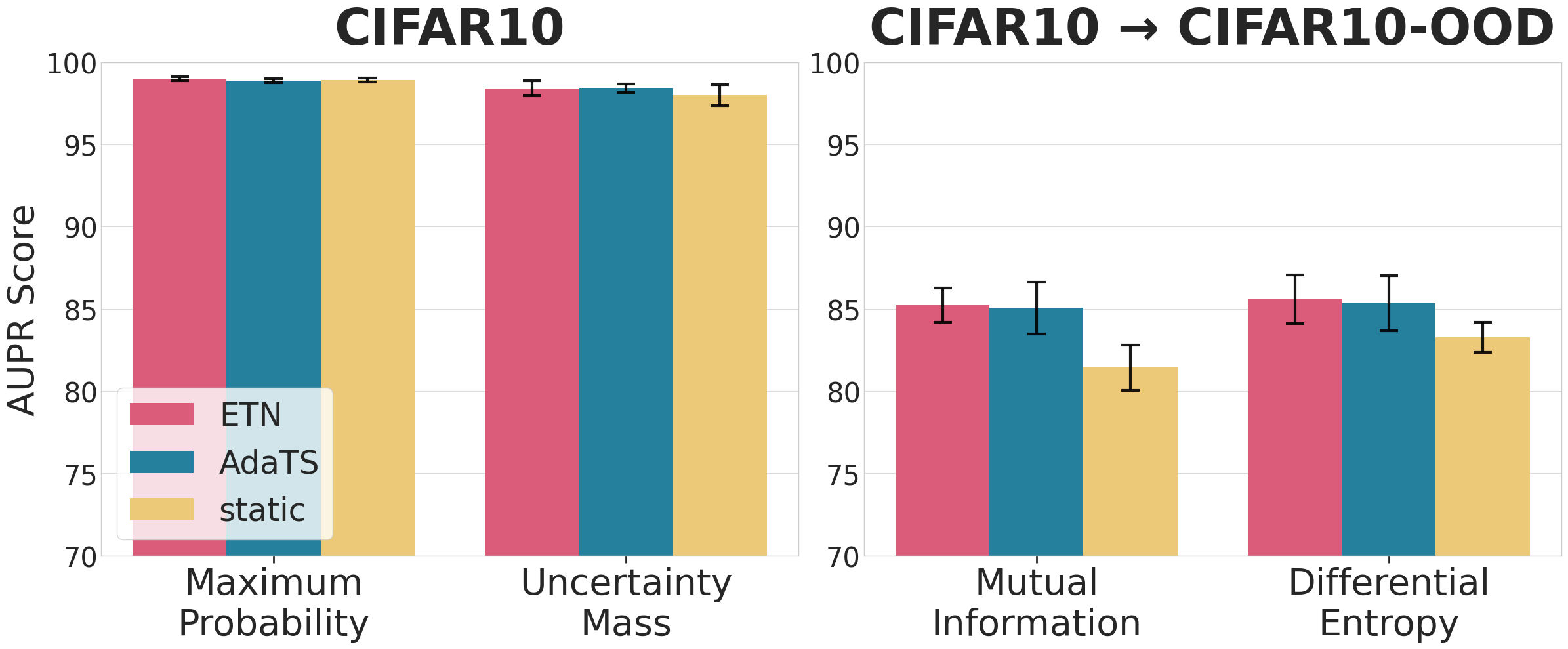}
    \end{minipage}\hfill
    \begin{minipage}{0.5\textwidth}
        \centering
        \includegraphics[width=\linewidth]{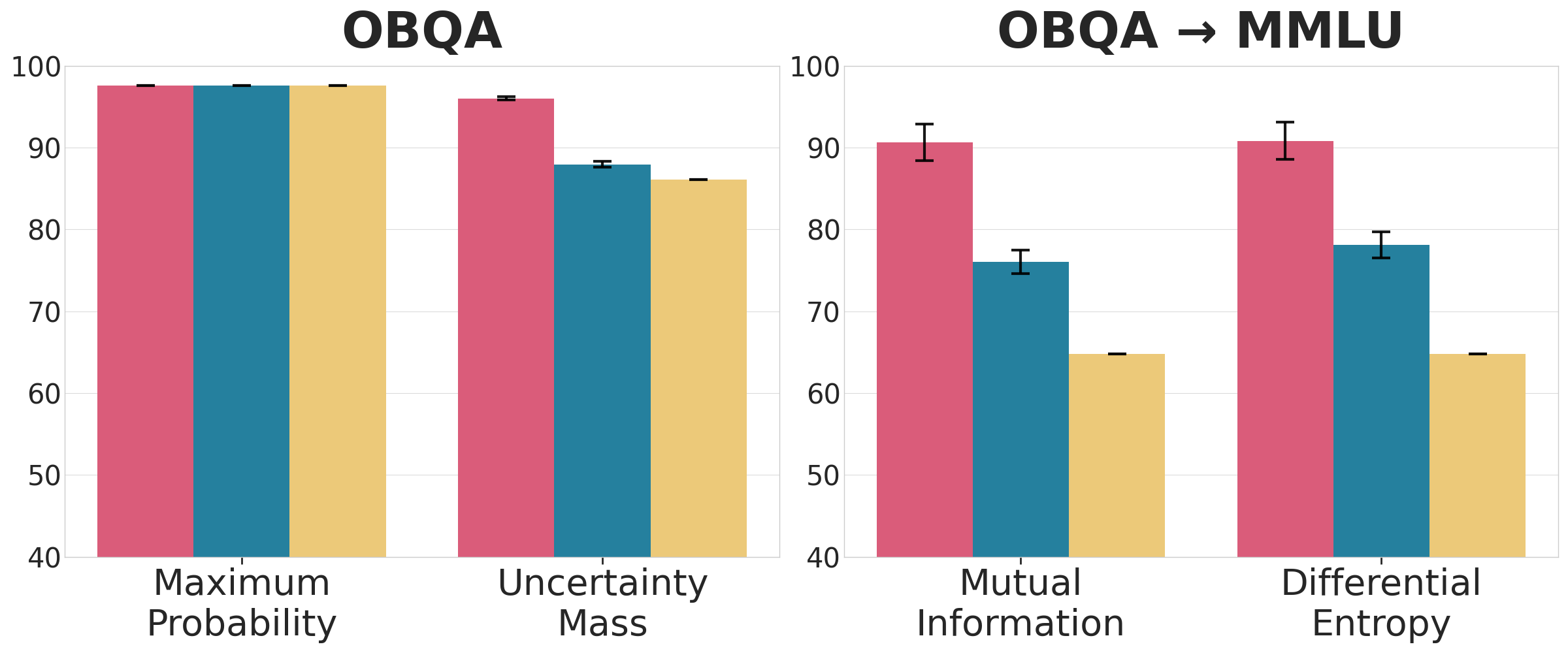}
    \end{minipage}\hfill
    \caption{Comparison of uncertainty estimation performance based on different transformation methods on CIFAR-10 and OBQA.}
    \label{fig:more_on_differnet_trans}
\end{figure*}

\begin{figure*}[ht]
    \centering
    \begin{minipage}{0.5\textwidth}
        \centering
        \includegraphics[width=\linewidth]{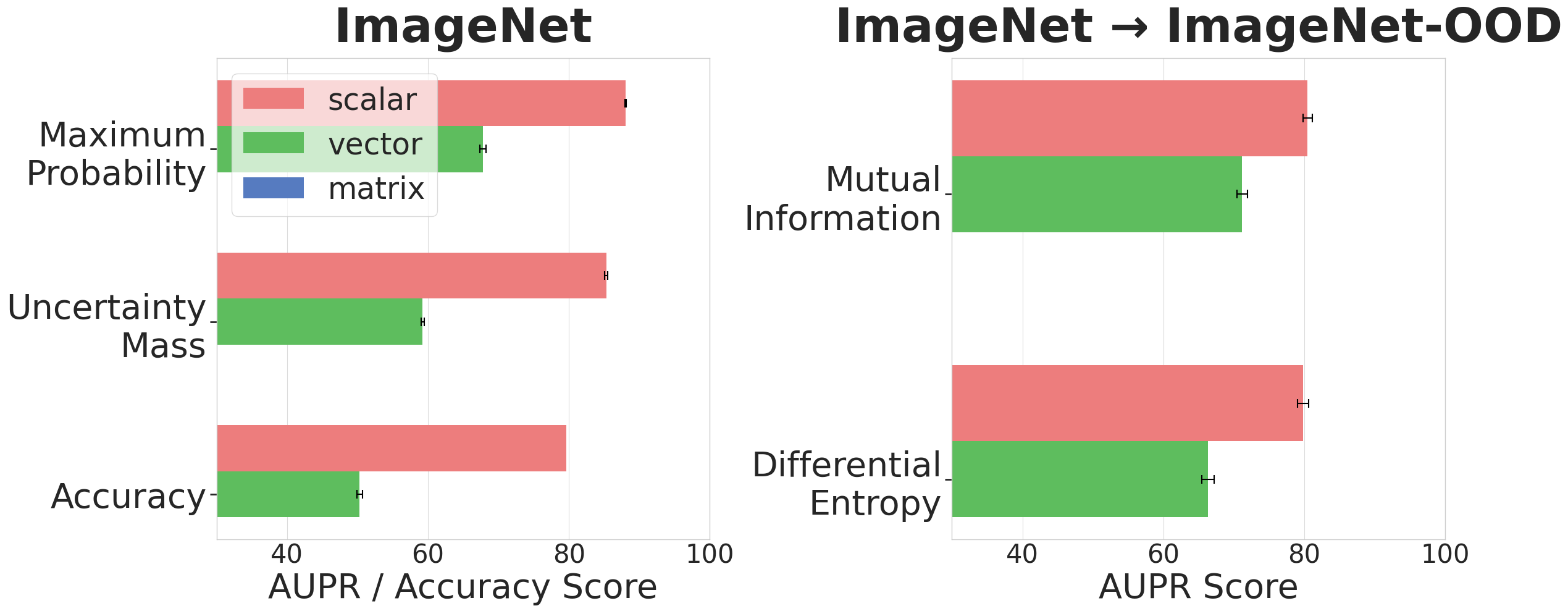}
    \end{minipage}\hfill
    \begin{minipage}{0.5\textwidth}
        \centering
        \includegraphics[width=\linewidth]{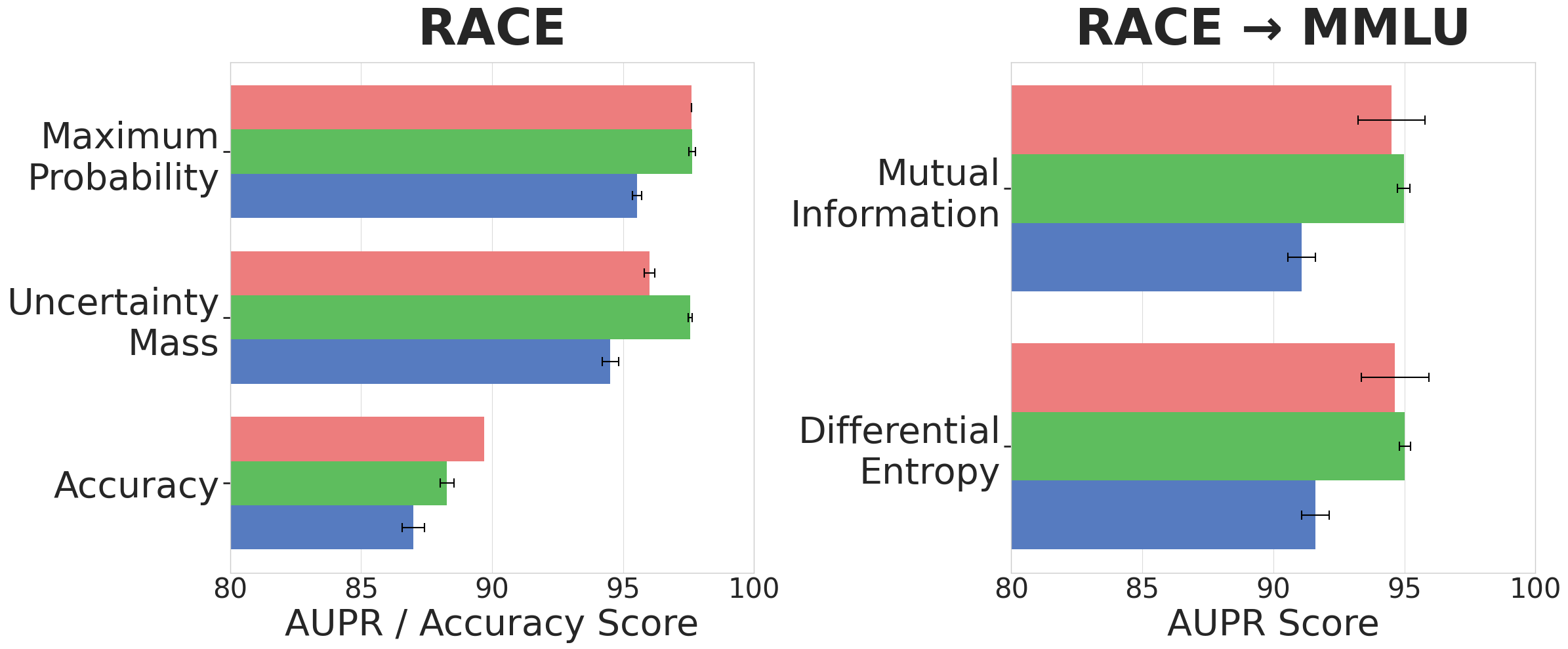}
    \end{minipage}\hfill
    \caption{Comparison of uncertainty estimation performance and accuracy across different dimensionalities of the transformation parameter $A$ modeled by ETN on ImageNet and RACE.}
\label{fig:more_on_different_dim}
\end{figure*}

\subsection{More on Comparison of Transformation Methods}
In this section, we present additional results on CIFAR-10 and OBQA, comparing different transformation strategies—specifically, static scaling and AdaTS. The results are shown in Figure~\ref{fig:more_on_differnet_trans}. Since CIFAR-10 is considerably smaller and simpler than the other datasets we evaluate, all three methods—static scaling, AdaTS, and ETN—achieve reasonably strong uncertainty estimation performance. AdaTS performs on par with ETN in terms of mutual information, while static scaling trails ETN by roughly $5\%$.

On OBQA, however, the differences between methods become more pronounced. Both static scaling and AdaTS exhibit substantially lower mutual information compared to ETN, with margins of approximately $13.6\%$ and $24.7\%$, respectively. These results highlight two key observations: (1) modeling sample-dependent transformation parameters is crucial for reliable uncertainty estimation, and (2) among sample-dependent approaches, our variational inference framework more effectively transforms logits to produce high-quality evidential uncertainty estimates.

\subsection{More on Comparison Across Transformation Dimensionalities}

In this section, we take a closer look at how the dimensionality of the transformation parameter $A$ affects uncertainty estimation performance across different transformation methods.

\paragraph{Results of ETN.}
We first analyze the behavior of ETN. For ImageNet, we exclude the matrix case since the corresponding covariance matrix would contain on the order of $10^{12}$ entries, which is intractable to store in GPU memory. The results are shown in Figure~\ref{fig:more_on_different_dim}.

On ImageNet, the scalar configuration outperforms the vector configuration for both confidence estimation and OOD detection. In contrast, on RACE, the vector configuration achieves the best performance on both ID and OOD metrics.

\paragraph{Results of static scaling.}
We next consider static scaling with different dimensionalities of $A$. The results are presented in Figure~\ref{fig:more_on_differnet_trans_with_dims_vector}.

For static scaling, all dimensionalities yield broadly similar uncertainty estimation performance on most datasets. An exception is ImageNet, where the maximum predicted probability tends to decrease as dimensionality increases, while OOD detection performance improves.

\paragraph{Results of AdaTS.}
Finally, we evaluate on AdaTS, with results summarized in Figure~\ref{fig:more_on_differnet_trans_with_dims_matrix}. In this case, higher-dimensional variants generally improve OOD detection compared to the scalar configuration. However, the behavior in confidence estimation is less consistent: maximum probability typically decreases while uncertainty mass increases, with ImageNet showing particularly irregular trends.

\paragraph{Discussion.}
Across ETN, static scaling, and AdaTS, a consistent trend emerges: increasing the dimensionality of the transformation parameter $A$ tends to degrade predictive accuracy and introduces a clear trade-off between OOD detection performance and core predictive capability. Moreover, none of the higher-dimensional variants—including the matrix formulation that operates directly in \emph{Dirichlet space}—surpasses scalar-based ETN across all datasets and metrics, with the sole exception of the maximum probability metric on ImageNet. Taken together, these results suggest that a simple scalar-based transformation within ETN offers the most effective and practical balance for adapting pretrained models to the EDL framework.

\subsection{AUPR Scores on Gemma-2-9B}
To further assess the robustness of ETN across different pretrained architectures, we evaluate its AUPR performance on Gemma-2-9B using OBQA and RACE. The results are shown in Table~\ref{tab:aupr_gemma}. Consistent with our findings on Llama-3.1, ETN delivers the strongest uncertainty estimation performance while fully preserving the model’s predictive accuracy. These results provide additional evidence that ETN generalizes effectively across diverse large-scale pretrained models.

\subsection{Uncertainty Estimation Performance Based on AUROC Scores}
Although recent works increasingly adopt AUPR as the primary metric for evaluating uncertainty estimation capability~\cite{deng2023uncertaintyestimationfisherinformationbased,chen2024r,yoon2024uncertainty}, we additionally report AUROC scores for image classification in Table~\ref{tab:auroc_image} and for LLMs in Table~\ref{tab:auroc_llm}.

On CIFAR-10, we observe that ETN outperforms all baselines across all AUROC-based metrics, showcasing its robustness across different uncertainty evaluation criteria.

For ImageNet, AUROC trends largely mirror those observed with AUPR. ETN remains competitive in OOD detection, while Laplace Approximation attains slightly higher AUROC for mutual information in some settings. In confidence estimation, we observe that DMM attains unusually high AUROC scores relative to its accuracy and AUPR. Upon inspecting its predictions, we find that DMM often produces nearly uniform predictive distributions with low $\alpha_0$, indicating uniformly high uncertainty across both ID and OOD inputs. This suggests that the inflated AUROC scores do not reflect reliable or informative confidence estimates.

For RACE and OBQA, ETN achieves the strongest OOD detection performance for both Llama-3.1 and Gemma-2. Moreover, in every setting, at least one confidence estimation metric (maximum probability or uncertainty mass) is maximized by ETN.

Although ETN is less dominant under AUROC than under AUPR, it remains competitive with strong baselines across all AUROC metrics while clearly outperforming them on our primary metric, AUPR. Overall, these results support ETN as an effective and practical method for uncertainty estimation in pretrained models.

\begin{figure}
    \centering
    \includegraphics[width=0.55\linewidth]{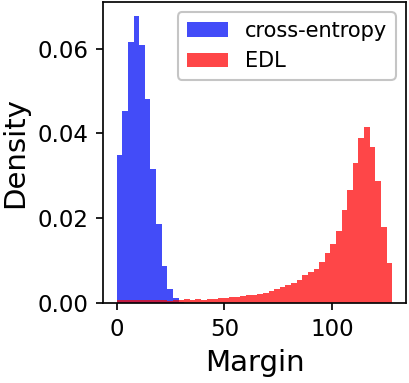}
    \caption{Histograms of logit margins for models trained with EDL and CE.}
    \label{fig:margin}
\end{figure}

\subsection{Empirical Comparison of Margins between EDL- and CE-Pretrained Models}
In this section, we empirically examine the logit margins of EDL- and CE-pretrained models to assess whether enlarging margins during the transformation process, as done by ETN, is a justified approach.
For this experiment, we use VGG16 on CIFAR-10. We compare a model trained from scratch with $\mathcal{L}_\text{EDL}$, where $\lambda = 0.01$ and $f(\cdot)=\mathrm{softplus}$, against a model trained from scratch with $\mathcal{L}_\text{CE}$. The remaining pretraining settings follow those in Table~\ref{tab:hyper_train}. Margins are computed on the CIFAR-10 training set.

Figure~\ref{fig:margin} shows the resulting margin histograms. The EDL-pretrained model exhibits larger margins than the CE-pretrained model. Together with Corollary~1, this result supports the validity of enlarging logit margins during the transformation process.

\begin{figure*}[ht]
    \centering
    \begin{minipage}{0.5\textwidth}
        \centering
        \includegraphics[width=\linewidth]{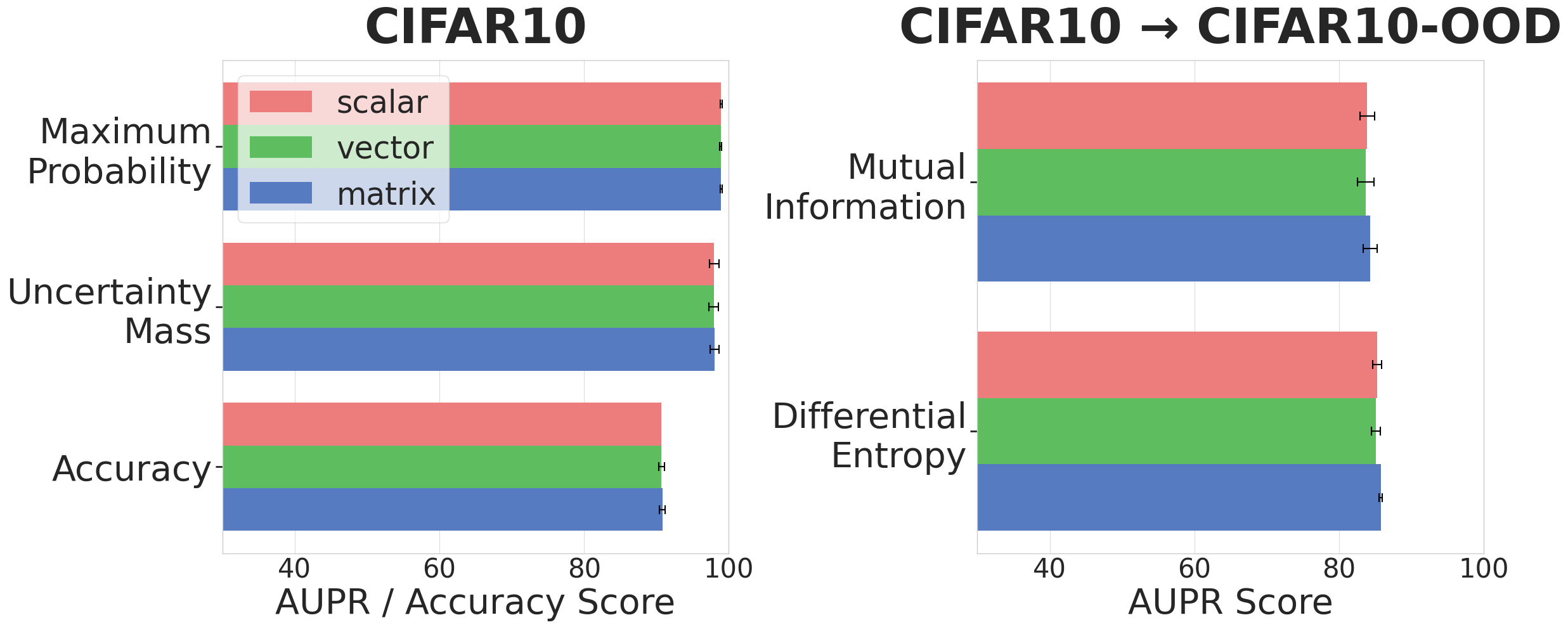}
    \end{minipage}\hfill
    \begin{minipage}{0.5\textwidth}
        \centering
        \includegraphics[width=\linewidth]{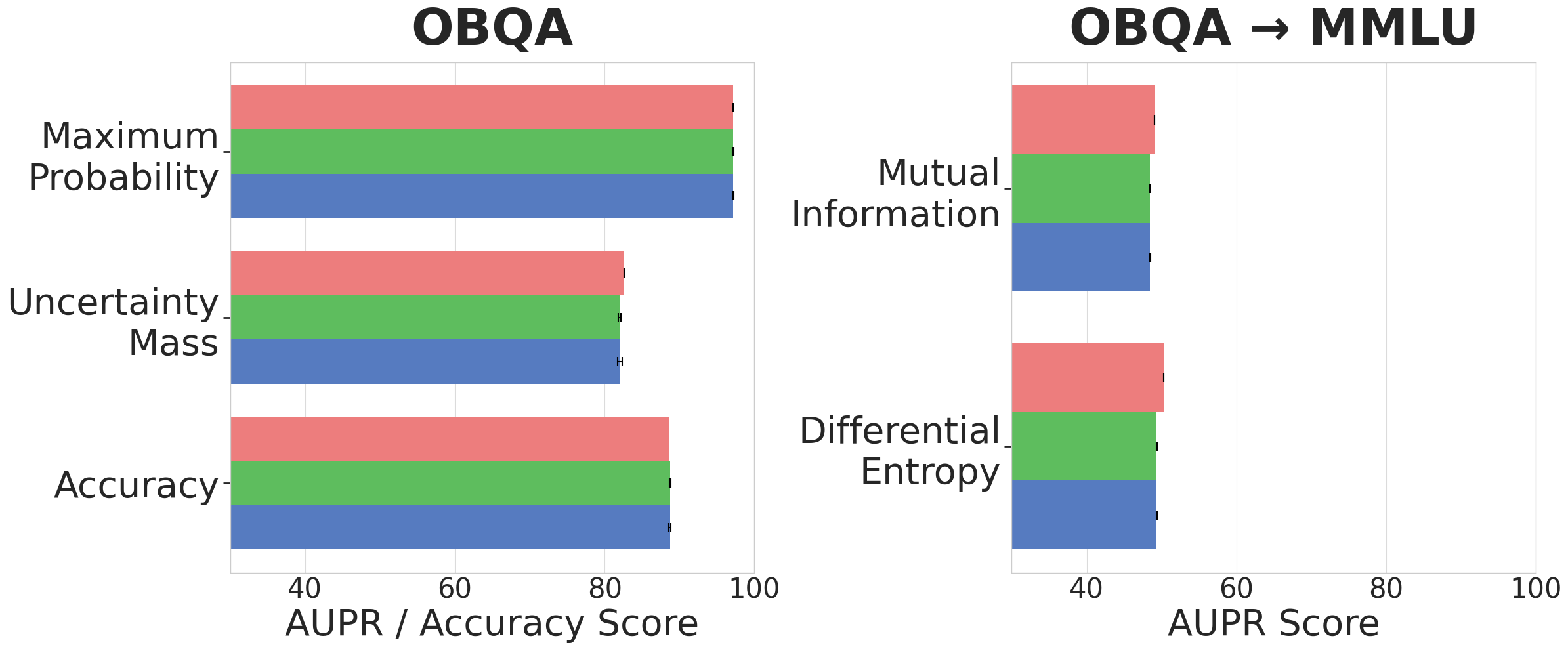}
    \end{minipage}\hfill
    \vspace{1em} 
    \begin{minipage}{0.5\textwidth}
        \centering
        \includegraphics[width=\linewidth]{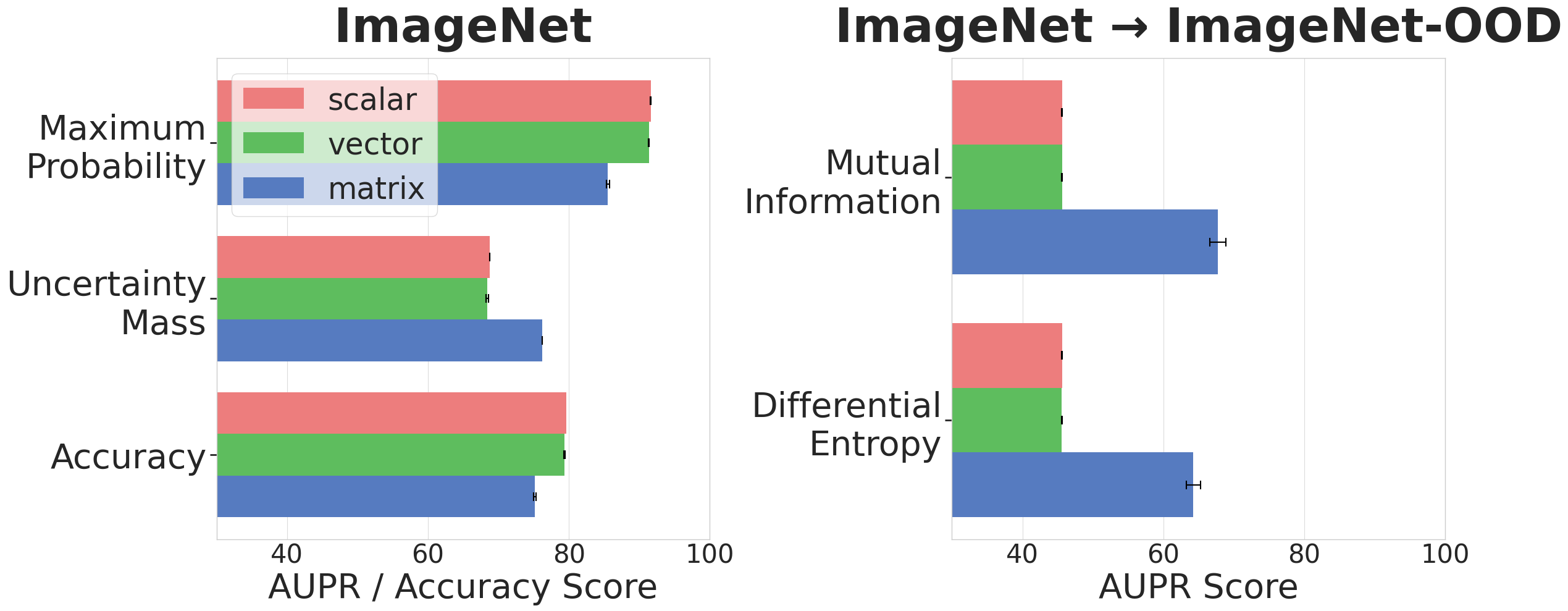}
    \end{minipage}\hfill
    \begin{minipage}{0.5\textwidth}
        \centering
        \includegraphics[width=\linewidth]{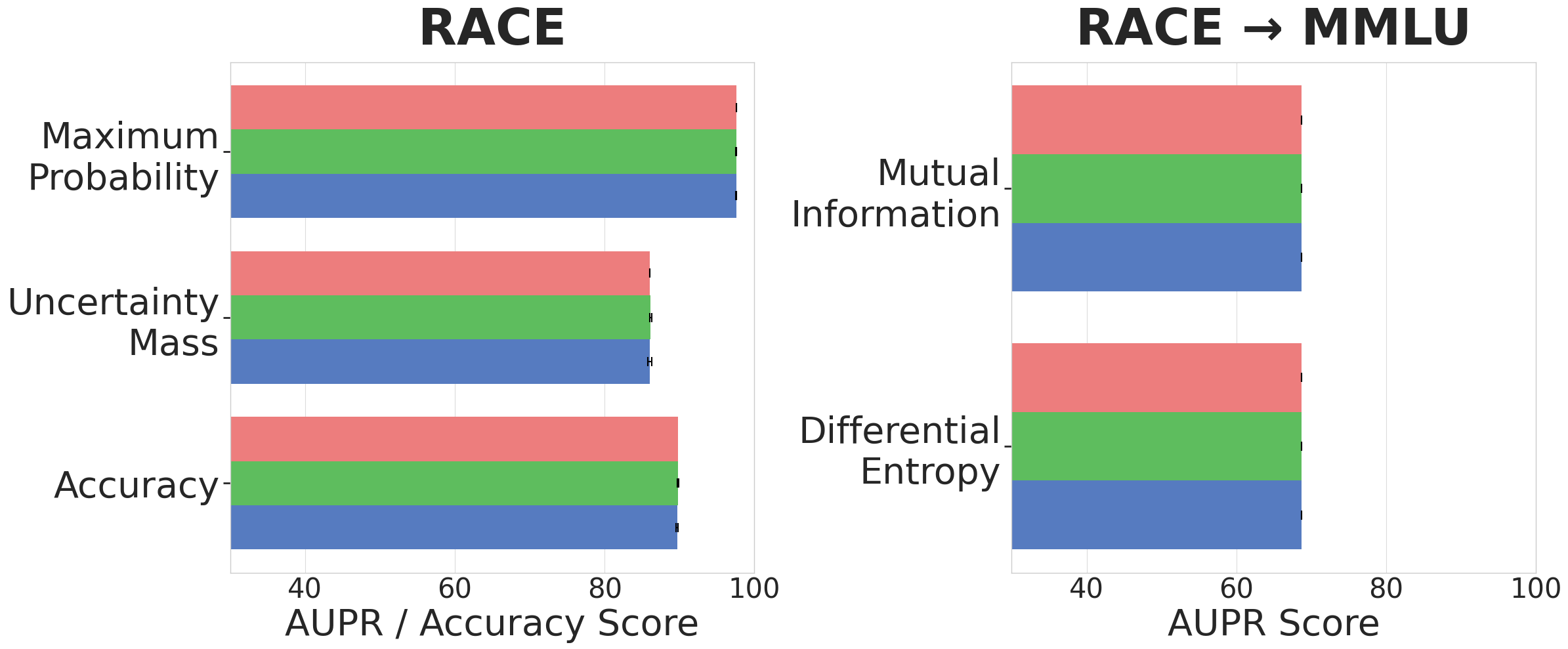}
    \end{minipage}\hfill
    \caption{Comparison of uncertainty estimation performance and accuracy across different dimensionalities of the transformation parameter $A$ modeled by \textbf{static} scaling.}
\label{fig:more_on_differnet_trans_with_dims_vector}
\end{figure*}

\begin{figure*}[ht]
    \centering
    \begin{minipage}{0.5\textwidth}
        \centering
        \includegraphics[width=\linewidth]{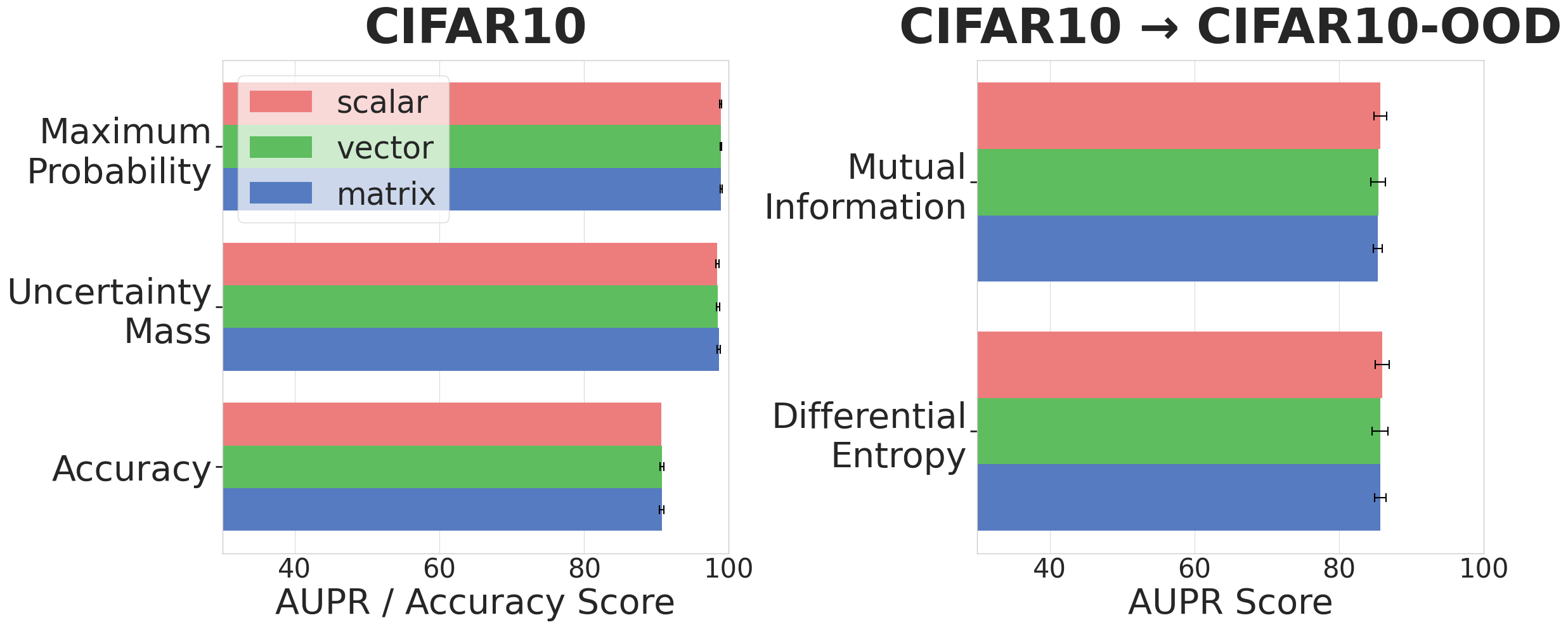}
    \end{minipage}\hfill
    \begin{minipage}{0.5\textwidth}
        \centering
        \includegraphics[width=\linewidth]{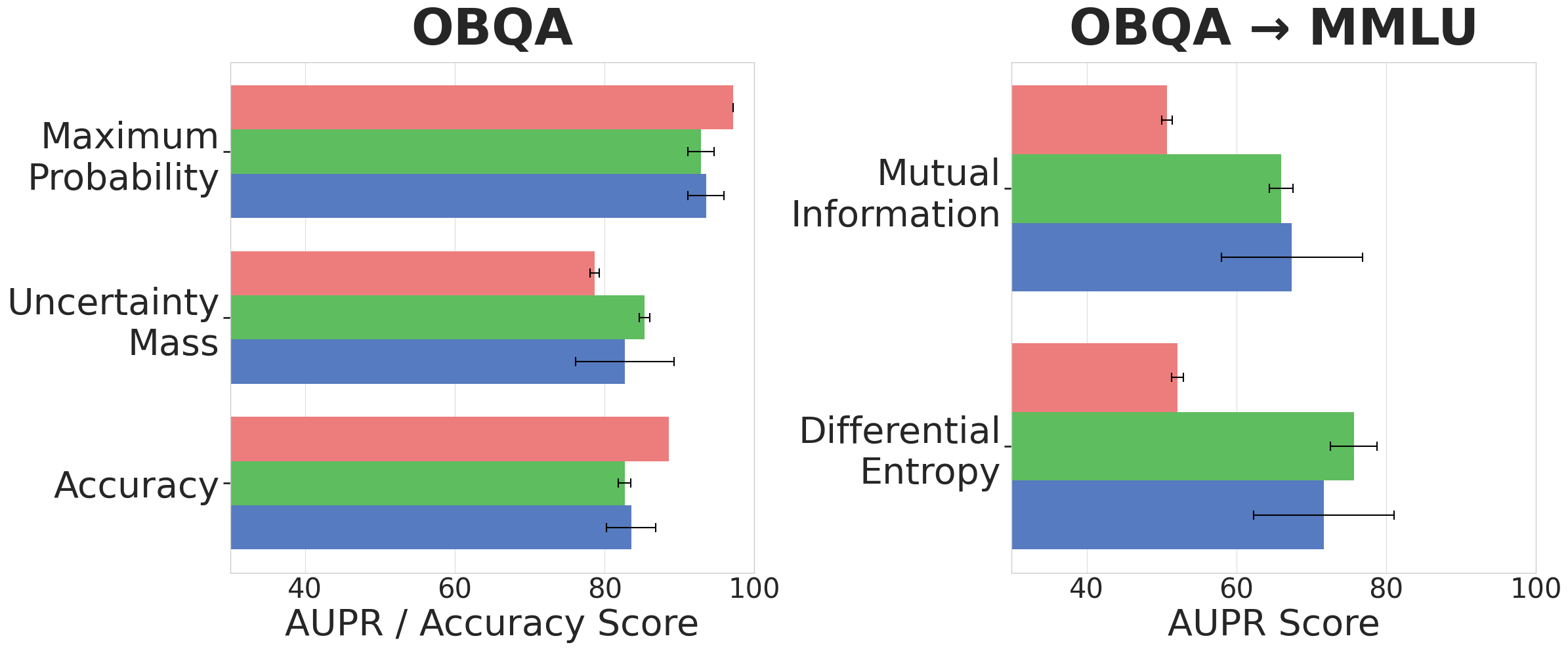}
    \end{minipage}\hfill
    \vspace{1em} 
    \begin{minipage}{0.5\textwidth}
        \centering
        \includegraphics[width=\linewidth]{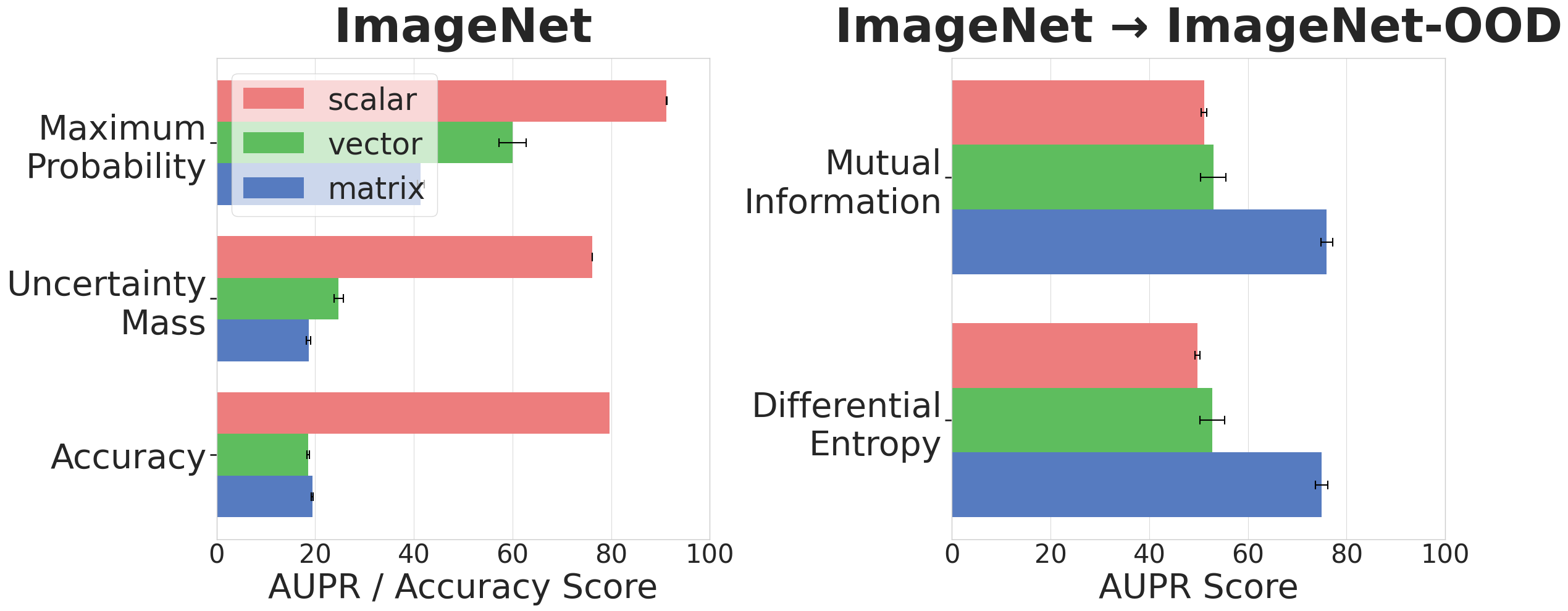}
    \end{minipage}\hfill
    \begin{minipage}{0.5\textwidth}
        \centering
        \includegraphics[width=\linewidth]{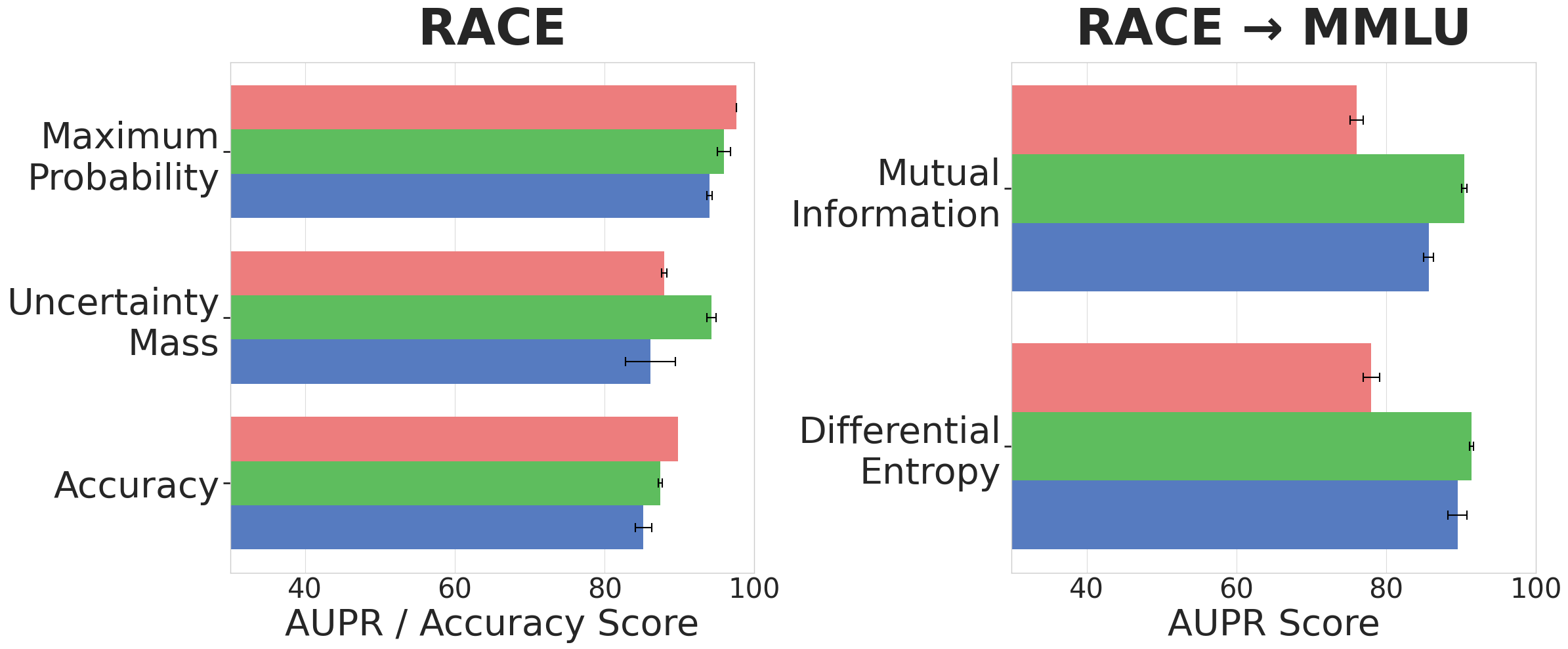}
    \end{minipage}\hfill
    \caption{Comparison of uncertainty estimation performance and accuracy across different dimensionalities of the transformation parameter $A$ modeled by \textbf{AdaTS}.}
\label{fig:more_on_differnet_trans_with_dims_matrix}
\end{figure*}
\begin{table*}[t]
\centering
\setlength{\tabcolsep}{5pt}

\resizebox{0.9\linewidth}{!}{
\begin{tabular}{lccc ccc ccc ccc}
\toprule
\multirow{2}{*}{Method}
  & \multicolumn{3}{c}{RACE}
  & \multicolumn{3}{c}{RACE $\to$ MMLU}
  & \multicolumn{3}{c}{OBQA}
  & \multicolumn{3}{c}{OBQA $\to$ MMLU} \\
\cmidrule(lr){2-4} \cmidrule(lr){5-7} \cmidrule(lr){8-10} \cmidrule(lr){11-13}
& ACC & MP & UM
& MP & MI & DE
& ACC & MP & UM
& MP & MI & DE \\
\midrule
$\mathrm{MAP}_{\text{CE}}$
  & \result{88.81}{0.2} & \result{98.16}{0.1} & --
  & \result{96.03}{0.1} & -- & -- 
  & \result{88.87}{0.6} & \result{97.59}{0.2} & --
  & \result{\underline{85.85}}{0.1} & -- & -- \\

DeepEns
  & \result{89.09}{0.1} 
  & \result{98.26}{0.0} 
  & --
  & \result{94.64}{0.2} 
  & \result{93.79}{0.4} 
  & -- 
  & \result{89.00}{0.0} 
  & \result{97.84}{0.3}  
  & -- 
  & \result{83.04}{0.9} 
  & \result{78.82}{0.5} 
  & -- \\
  
MCD
  & \result{89.38}{0.4} 
  & \result{\underline{98.40}}{0.2} 
  & --
  & \result{\underline{96.36}}{0.3} 
  & \result{\textbf{95.36}}{0.4} 
  & -- 
  & \result{89.60}{0.5} 
  & \result{\underline{97.87}}{0.5} 
  & -- 
  & \result{85.59}{1.4} 
  & \result{\underline{81.81}}{1.8} 
  & -- \\

LL
  & \result{88.51}{0.2} 
  & \result{69.05}{0.3} 
  & --
  & \result{27.47}{1.4} 
  & \result{27.03}{1.0} 
  & -- 
  & \result{92.53}{0.2} 
  & \result{79.45}{0.6} 
  & --  
  & \result{25.32}{0.6} 
  & \result{26.25}{0.8} 
  & -- \\

$\mathrm{MAP}_{\text{EDL}}$
  & \result{87.14}{0.2} 
  & \result{92.27}{0.7} 
  & \result{87.60}{1.1} 
  & \result{93.97}{1.5} 
  & \result{83.89}{1.1} 
  & \result{90.36}{1.2}
  & \result{87.53}{0.4} 
  & \result{94.63}{0.9} 
  & \result{86.32}{3.3} 
  & \result{80.28}{2.7} 
  & \result{63.50}{1.5} 
  & \result{79.78}{2.6} \\

DMM
  & \result{\underline{89.42}}{0.2} 
  & \result{97.96}{0.1} 
  & \result{\underline{95.20}}{0.4}
  & \result{93.84}{0.3} 
  & \result{92.77}{0.6} 
  & \result{\underline{93.80}}{0.3}
  & \result{\underline{93.00}}{0.6} 
  & \result{96.68}{0.0} 
  & \result{\underline{95.44}}{0.1} 
  & \result{81.48}{0.5} 
  & \result{80.53}{0.4} 
  & \result{\underline{81.51}}{0.5} \\

IB-EDL
  & \result{87.98}{0.3} 
  & \result{93.72}{0.1} 
  & \result{88.04}{0.4}
  & \result{90.67}{0.9} 
  & \result{89.43}{1.5} 
  & \result{90.70}{1.1}
  & \result{85.13}{0.4} 
  & \result{94.60}{0.4} 
  & \result{92.15}{0.1} 
  & \result{79.99}{0.8} 
  & \result{66.58}{2.1} 
  & \result{74.28}{1.7} \\
\rowcolor{gray!15}
ETN
  & \result{\textbf{89.48}}{0.0} 
  & \result{\textbf{98.43}}{0.0} 
  & \result{\textbf{95.94}}{0.1} 
  & \result{\textbf{96.70}}{0.0} 
  & \result{\underline{95.21}}{0.3} 
  & \result{\textbf{95.29}}{0.3}
  & \result{\textbf{93.20}}{0.0} 
  & \result{\textbf{98.06}}{0.0} 
  & \result{\textbf{96.00}}{0.1} 
  & \result{\textbf{86.95}}{0.1} 
  & \result{\textbf{82.69}}{0.2} 
  & \result{\textbf{84.39}}{0.2} \\
\bottomrule
\end{tabular}
}
\caption{AUPR scores of Gemma-2-9B on RACE and OBQA, using MMLU subsets as OOD data.}
\label{tab:aupr_gemma}
\end{table*}

\begin{table*}[t]
\centering
\setlength{\tabcolsep}{5pt}

\begin{subtable}[t]{\textwidth}
\centering
\resizebox{0.7\linewidth}{!}{
\begin{tabular}{lcc ccc ccc}
\toprule
\multirow{2}{*}{Method}
  & \multicolumn{2}{c}{CIFAR-10}
  & \multicolumn{3}{c}{CIFAR-10 $\to$ SVHN}
  & \multicolumn{3}{c}{CIFAR-10 $\to$ CIFAR-100} \\
\cmidrule(lr){2-3} \cmidrule(lr){4-6} \cmidrule(lr){7-9}
& MP & UM
& MP & MI & DE
& MP & MI & DE \\
\midrule
$\mathrm{MAP}_{\text{CE}}$
  & \result{87.71}{0.1} & --
  & \result{73.26}{5.6} & -- & --
  & \result{79.10}{2.3} & -- & -- \\

DeepEns
  & \result{\underline{90.78}}{0.5} & --
  & \result{\underline{87.43}}{1.1} & \result{49.19}{0.2} & --
  & \result{\underline{83.63}}{0.8} & \result{49.39}{0.3} & -- \\
  
MCD
  & \result{77.67}{5.4} & --
  & \result{73.26}{4.2} & \result{78.56}{1.4} & --
  & \result{68.72}{6.2} & \result{75.62}{2.5} & -- \\

LA
  & \result{90.23}{0.3} & --
  & \result{84.22}{0.1} & \result{\underline{83.46}}{0.2} & --
  & \result{82.38}{0.2} & \result{81.93}{0.2} & -- \\

$\mathrm{MAP}_{\text{EDL}}$
  & \result{87.91}{0.5} & \result{\underline{86.98}}{0.8}
  & \result{81.61}{0.7} & \result{81.58}{3.1} & \result{81.93}{1.9}
  & \result{80.27}{0.7} & \result{\underline{82.24}}{0.4} & \result{81.89}{0.4} \\

DMM
  & \result{90.65}{1.2} & \result{83.36}{1.9}
  & \result{85.68}{2.6} & \result{80.20}{9.7} & \result{\underline{84.99}}{6.6}
  & \result{79.25}{3.0} & \result{79.69}{5.5} & \result{\underline{82.43}}{4.0} \\

IB-EDL
  & \result{87.58}{1.3} & \result{86.19}{1.0}
  & \result{78.22}{2.9} & \result{76.72}{3.0} & \result{77.64}{3.0}
  & \result{79.39}{1.7} & \result{79.41}{1.2} & \result{79.76}{1.5} \\
\rowcolor{gray!15}
ETN
  & \result{\textbf{91.80}}{0.9} & \result{\textbf{87.98}}{2.8}
  & \result{\textbf{88.60}}{1.4} & \result{\textbf{88.40}}{0.9} & \result{\textbf{88.79}}{1.2}
  & \result{\textbf{84.40}}{0.5} & \result{\textbf{84.50}}{0.6} & \result{\textbf{84.84}}{0.5} \\
\bottomrule
\end{tabular}}
\end{subtable}

\vspace{0.75em}
\begin{subtable}[t]{\textwidth}
\centering
\resizebox{0.9\linewidth}{!}{
\begin{tabular}{lcc ccc ccc ccc}
\toprule
\multirow{2}{*}{Method}
  & \multicolumn{2}{c}{ImageNet}
  & \multicolumn{3}{c}{ImageNet $\to$ ImageNet-A}
  & \multicolumn{3}{c}{ImageNet $\to$ ImageNet-S}
  & \multicolumn{3}{c}{ImageNet $\to$ ImageNet-R} \\
\cmidrule(lr){2-3} \cmidrule(lr){4-6} \cmidrule(lr){7-9} \cmidrule(lr){10-12}
& MP & UM
& MP & MI & DE
& MP & MI & DE
& MP & MI & DE \\
\midrule
$\mathrm{MAP}_{\text{CE}}$
  & \result{80.28}{0.2} & --
  & \result{70.51}{0.2} & -- & -- 
  & \result{67.78}{2.9} & -- & -- 
  & \result{67.43}{1.0} & -- & -- \\

DeepEns
  & \result{63.83}{0.4} & --
  & \result{53.39}{2.3} & \result{49.63}{0.1} & -- 
  & \result{26.96}{2.7} & \result{50.15}{0.3} & -- 
  & \result{54.25}{1.9} & \result{50.01}{0.1} & -- \\
  
MCD
  & \result{78.90}{0.4} & --
  & \result{67.64}{1.6} & \result{50.17}{0.2} & -- 
  & \result{60.11}{3.5} & \result{50.12}{0.0} & -- 
  & \result{66.94}{0.8} & \result{50.02}{0.1} & -- \\

LA
  & \result{49.16}{1.5} & --
  & \result{\underline{78.39}}{0.1} & \result{\textbf{83.55}}{0.0} & -- 
  & \result{\underline{75.58}}{0.1} & \result{\textbf{81.30}}{0.0} & -- 
  & \result{\underline{75.39}}{0.1} & \result{\textbf{79.97}}{0.0} & -- \\

$\mathrm{MAP}_{\text{EDL}}$
  & \result{72.72}{0.3} & \result{55.10}{0.3}
  & \result{70.80}{0.8} & \result{66.07}{1.1} & \result{\underline{66.49}}{1.1}
  & \result{68.83}{5.3} & \result{69.50}{4.6} & \result{\underline{69.90}}{4.7}
  & \result{68.11}{2.3} & \result{71.41}{2.4} & \result{\underline{71.55}}{2.4} \\

DMM
  & \result{\textbf{92.54}}{0.4} & \result{\textbf{91.86}}{0.4}
  & \result{48.30}{0.4} & \result{48.38}{0.4} & \result{48.36}{0.3}
  & \result{43.48}{3.6} & \result{43.82}{3.4} & \result{43.77}{3.4}
  & \result{48.15}{2.0} & \result{48.24}{1.9} & \result{48.25}{1.9} \\

IB-EDL
  & \result{\underline{91.05}}{0.1} & \result{48.70}{0.1}
  & \result{55.50}{1.8} & \result{50.03}{0.2} & \result{48.84}{0.3}
  & \result{54.65}{1.6} & \result{49.00}{0.7} & \result{49.84}{0.2}
  & \result{55.03}{1.4} & \result{49.63}{0.3} & \result{49.42}{0.1} \\
  
\rowcolor{gray!15}
ETN
  & \result{69.32}{0.3} & \result{\underline{64.78}}{0.4}
  & \result{\textbf{81.81}}{0.2} & \result{\underline{77.21}}{0.1} & \result{\textbf{76.35}}{0.9}
  & \result{\textbf{78.34}}{0.6} & \result{\underline{74.08}}{0.9} & \result{\textbf{73.44}}{1.4}
  & \result{\textbf{79.42}}{0.7} & \result{\underline{75.63}}{1.0} & \result{\textbf{74.88}}{1.6} \\
\bottomrule
\end{tabular}}
\end{subtable}
\caption{AUROC scores on CIFAR-10, SVHN, and CIFAR-100 (top), and on ImageNet, ImageNet-A, ImageNet-S, and ImageNet-R (bottom).}
\label{tab:auroc_image}
\end{table*}

\begin{table*}[t]
\centering
\setlength{\tabcolsep}{5pt}

\begin{subtable}[t]{\textwidth}
\centering
\resizebox{0.9\linewidth}{!}{
\begin{tabular}{lcc ccc cc ccc}
\toprule
\multirow{2}{*}{Method}
  & \multicolumn{2}{c}{RACE}
  & \multicolumn{3}{c}{RACE $\to$ MMLU}
  & \multicolumn{2}{c}{OBQA}
  & \multicolumn{3}{c}{OBQA $\to$ MMLU} \\
\cmidrule(lr){2-3} \cmidrule(lr){4-6} \cmidrule(lr){7-8} \cmidrule(lr){9-11}
& MP & UM
& MP & MI & DE
& MP & UM
& MP & MI & DE \\
\midrule
$\mathrm{MAP}_{\text{CE}}$
  & \result{\textbf{87.59}}{0.4} & --
  & \result{\underline{87.02}}{0.9} & -- & -- 
  & \result{81.94}{0.9} & --
  & \result{\underline{83.48}}{1.3} & -- & -- \\

DeepEns
  & \result{86.03}{0.5} & --
  & \result{80.76}{0.8} & \result{73.49}{2.0} & -- 
  & \result{80.61}{0.5} & --
  & \result{80.33}{1.4} & \result{75.67}{0.8} & -- \\
  
MCD
  & \result{\underline{87.11}}{0.0} & --
  & \result{86.39}{0.01} & \result{\textbf{85.17}}{0.8} & -- 
  & \result{\underline{83.02}}{0.2} & --
  & \result{81.96}{0.0} & \result{69.58}{0.5} & -- \\

LL
  & \result{45.04}{0.8} & --
  & \result{49.82}{1.9} & \result{53.43}{1.3} & -- 
  & \result{41.32}{0.2} & --  
  & \result{47.43}{0.7} & \result{46.37}{0.6} & -- \\

$\mathrm{MAP}_{\text{EDL}}$
  & \result{70.10}{0.7} & \result{61.71}{2.2} 
  & \result{70.22}{0.9} & \result{65.15}{1.5} & \result{70.16}{0.9}
  & \result{64.44}{1.7} & \result{48.60}{0.2}
  & \result{72.17}{1.7} & \result{66.99}{2.2} & \result{72.53}{1.9} \\

DMM
  & \result{82.34}{0.4} & \result{75.21}{2.8}
  & \result{77.88}{2.0} & \result{76.85}{0.7} & \result{78.89}{1.6}
  & \result{77.12}{1.1} & \result{\underline{69.32}}{3.7}
  & \result{78.98}{0.6} & \result{\underline{73.51}}{2.5} & \result{78.47}{0.4} \\

IB-EDL
  & \result{82.85}{0.4} & \result{\underline{77.21}}{0.6}
  & \result{80.48}{1.6} & \result{66.86}{2.9} & \result{\underline{80.17}}{1.7}
  & \result{79.98}{0.6} & \result{63.71}{0.9}
  & \result{82.83}{1.2} & \result{67.89}{2.1} & \result{\underline{83.16}}{1.2} \\
\rowcolor{gray!15}
ETN
  & \result{84.96}{0.0} & \result{\textbf{77.69}}{0.5}
  & \result{\textbf{87.57}}{0.0} & \result{\underline{81.42}}{3.0} & \result{\textbf{81.81}}{3.1}
  & \result{\textbf{83.21}}{0.0} & \result{\textbf{73.60}}{1.0}
  & \result{\textbf{87.09}}{0.0} & \result{\textbf{85.30}}{1.2} & \result{\textbf{86.27}}{0.9} \\
\bottomrule
\end{tabular}
}
\end{subtable}

\vspace{0.75em}

\begin{subtable}[t]{\textwidth}
\centering
\resizebox{0.9\linewidth}{!}{
\begin{tabular}{lcc ccc cc ccc}
\toprule
\multirow{2}{*}{Method}
  & \multicolumn{2}{c}{RACE}
  & \multicolumn{3}{c}{RACE $\to$ MMLU}
  & \multicolumn{2}{c}{OBQA}
  & \multicolumn{3}{c}{OBQA $\to$ MMLU} \\
\cmidrule(lr){2-3} \cmidrule(lr){4-6} \cmidrule(lr){7-8} \cmidrule(lr){9-11}
& MP & UM
& MP & MI & DE
& MP & UM
& MP & MI & DE \\
\midrule
$\mathrm{MAP}_{\text{CE}}$
  & \result{89.39}{0.3} & -- 
  & \result{84.74}{0.2} & -- & -- 
  & \result{81.35}{0.2} & -- 
  & \result{\underline{83.48}}{1.3} & -- & -- \\
  
DeepEns
  & \result{89.41}{0.4} & --
  & \result{80.13}{0.7} & \result{77.11}{1.1} & -- 
  & \result{\textbf{88.03}}{1.7} & --
  & \result{75.44}{1.2} & \result{67.04}{1.2} & -- \\
  
MCD
  & \result{\underline{90.45}}{0.4} & -- 
  & \result{\underline{85.50}}{1.0} & \result{\underline{81.61}}{0.8} & -- 
  & \result{81.83}{1.8} & -- 
  & \result{78.99}{1.5} & \result{\underline{75.31}}{1.7} & -- \\

LL
  & \result{36.16}{0.8} & -- 
  & \result{52.36}{1.8} & \result{51.37}{1.8} & -- 
  & \result{29.94}{1.0} & -- 
  & \result{47.43}{0.7} & \result{51.54}{0.7} & -- \\

$\mathrm{MAP}_{\text{EDL}}$
  & \result{71.43}{3.0} & \result{61.76}{2.0}
  & \result{80.65}{3.4} & \result{64.75}{3.9} & \result{77.12}{3.4}
  & \result{76.90}{2.5} & \result{56.79}{9.8}
  & \result{73.38}{3.7} & \result{59.13}{1.5} & \result{73.41}{3.9} \\

DMM
  & \result{88.96}{0.4} & \result{\textbf{84.28}}{1.6}
  & \result{79.72}{0.7} & \result{76.31}{1.4} & \result{\underline{79.48}}{0.7}
  & \result{77.97}{1.5} & \result{\textbf{75.12}}{2.2}
  & \result{78.98}{0.6} & \result{73.51}{2.5} & \result{\underline{78.47}}{0.4} \\

IB-EDL
  & \result{78.58}{1.0} & \result{62.83}{0.6}
  & \result{74.67}{1.5} & \result{68.19}{3.5} & \result{72.94}{2.4}
  & \result{80.67}{0.7} & \result{\underline{69.30}}{0.5}
  & \result{73.78}{1.0} & \result{61.76}{2.8} & \result{68.66}{2.1} \\
\rowcolor{gray!15}
ETN
  & \result{\textbf{91.30}}{0.0} & \result{\underline{80.97}}{1.2}
  & \result{\textbf{87.00}}{0.0} & \result{\textbf{83.12}}{0.8} & \result{\textbf{83.31}}{0.8}
  & \result{\underline{82.68}}{0.0} & \result{64.87}{0.7}
  & \result{\textbf{85.67}}{0.0} & \result{\textbf{78.82}}{0.3} & \result{\textbf{80.92}}{0.3} \\
\bottomrule
\end{tabular}
}
\end{subtable}
\caption{AUROC scores on on RACE and OBQA, using MMLU subsets as OOD data. Results on Llama-3.1-8B is reported (top), and Gemma-2-9B is reported (bottom).}
\label{tab:auroc_llm}
\end{table*}

\clearpage
\clearpage

\end{document}